\documentclass[twoside]{article}

\usepackage[accepted]{aistats2022}
\usepackage[titletoc]{appendix}
\usepackage{subcaption}
\usepackage{varwidth}
\usepackage{capt-of}
\usepackage{boldline}
\usepackage{epsfig}
\usepackage{latexsym}
\usepackage{url}
\usepackage{float}
\usepackage[hypertexnames=false,hyperfootnotes=false]{hyperref}
\usepackage{texnansi}
\usepackage{color}
\usepackage{tikz}
\usepackage{afterpage}
\usepackage{enumerate}
\usepackage[normalem]{ulem}
\usepackage{lmodern}
\usepackage{rotating}
\usepackage{graphicx}
\usepackage{multicol}
\usepackage{booktabs}
\usepackage{ifthen}
\usepackage{enumitem}
\setlist[enumerate]{topsep=0pt,itemsep=-1ex,partopsep=1ex,parsep=1ex}

\usepackage{amsmath}
\usepackage{amsthm}
\usepackage{amssymb}
\usepackage{amsfonts}
\usepackage{titlesec}

\makeatletter
\g@addto@macro{\UrlBreaks}{\UrlOrds}
\makeatother

\newcommand{\field}[1]{\ensuremath{\mathbb{#1}}}
\newcommand{\R}{\ensuremath{\field{R}}} % real numbers
\newcommand{\Z}{\ensuremath{\field{Z}}} % integers
\newcommand{\Mscr}{\ensuremath{\mathcal M}}

\newcommand{\minimize}{\ensuremath{\mathop{\mathrm{minimize}}\limits}}
\newcommand{\maximize}{\ensuremath{\mathop{\mathrm{maximize}}\limits}}
\newtheorem{assumption}{Assumption}

\newtheorem{theorem}{Theorem}
\newtheorem{corollary}{Corollary}
\newtheorem{lemma}{Lemma}
\newtheorem{claim}{Claim}

\newtheorem{informal-theorem}{Theorem}
\newtheorem{informal-proposition}{Proposition}
\newcommand{\calN}{\mathcal{N}}
\newcommand{\calS}{\mathcal{S}}

\newcommand{\shelf}{\Mscr}

\renewcommand{\Pr}{\mathbb{P}}

\usetikzlibrary{arrows,patterns,plotmarks}
\tikzstyle{every picture} += [>=stealth]
\makeatletter
\def\@seccntformat#1{\csname the#1\endcsname.\quad}
\makeatother

\def\draft{1}  % 1 for draft version (eg including author notes), 0 for non-draft version

\usepackage{dsfont}
\usepackage{tcolorbox}

\usepackage{mathtools}

\usepackage{algorithm}
\usepackage[noend]{algpseudocode}
\usepackage[normalem]{ulem}

\ifnum\draft=1

\else

\fi

\def\<{\langle}
\def\>{\rangle}

\def\R{\mathbb{R}}

\newcommand{\norm}[1]{\left\lVert\mspace{1mu} #1 \mspace{1mu}\right\rVert}

\newcommand{\E}[1] {{\mathbb{E}}\left(#1\right)}

\newcommand {\Prob}[1] {{\mathbb{P}}\left(#1\right)}

\newcommand {\F } {{\mathrm{F}}}

\newcommand {\Ber} {{\rm Ber}}
\newcommand {\1}[1] {{\mathds{1}}\left\{#1\right\}}

\newcommand{\rank}{\mathrm{rank}}

\renewcommand{\O}{\mathrm{O}}

\newcommand{\dd}{\mathrm{d}}
\newcommand{\Dbalan}{\Delta_{\mathrm{balancing}}}

\newcommand{\sigmaL}{L\log(n)}
\newcommand{\noise}{\frac{\sigmaL}{\sigma_{\min}}\sqrt{\frac{n}{p}}}
\newcommand{\lognoise}{\frac{\sigmaL}{\sigma_{\min}}\sqrt{\frac{n\log(n)}{p}}}

\newcommand{\bXdT}{\bar{X}^{\dd T}}
\newcommand{\bXd}{\bar{X}^{\dd}}
\newcommand{\bYdT}{\bar{Y}^{\dd T}}
\newcommand{\bYd}{\bar{Y}^{\dd}}

\DeclareMathOperator*{\argmin}{arg\,min}

\usepackage[capitalise]{cleveref}

\usepackage{natbib}
 \bibpunct[, ]{(}{)}{,}{a}{}{,}%
\begin{document}

\runningtitle{Uncertainty Quantification With Heterogeneous and Sub-Exponential Noise}


\twocolumn[

\aistatstitle{Uncertainty Quantification For Low-Rank Matrix Completion With Heterogeneous and Sub-Exponential Noise}

\aistatsauthor{ Vivek F. Farias \And Andrew A. Li \And  Tianyi Peng }

%\aistatsaddress{ Sloan School of Management, Massachusetts Institute of Technology, Cambridge, MA 02139 \And  Tepper School of Business, Carnegie Mellon University, Pittsburgh, PA 15213 \And Department of Aeronautics and Astronautics, Massachusetts Institute of Technology, Cambridge, MA 02139 } ]

\aistatsaddress{Massachusetts Institute of Technology \And Carnegie Mellon University \And Massachusetts Institute of Technology}]

\begin{abstract}
The problem of low-rank matrix completion with heterogeneous and sub-exponential (as opposed to homogeneous and Gaussian) noise is particularly relevant to a number of applications in modern commerce. Examples include panel sales data and data collected from web-commerce systems such as recommendation engines. An important unresolved question for this problem is characterizing the distribution of estimated matrix entries under common low-rank estimators. Such a characterization is essential to any application that requires quantification of uncertainty in these estimates and has heretofore only been available under the assumption of homogenous Gaussian noise. Here we characterize the distribution of estimated matrix entries when the observation noise is heterogeneous sub-exponential and provide, as an application, explicit formulas for this distribution when observed entries are Poisson or Binary distributed.  
\end{abstract}

\section{Introduction} \label{sec:intro}

\begin{table*}[ht!]%\hspace{-1mm}
%\vspace{2.5em}
\begin{center}
%\small{
\begin{tabular}{cc}
\toprule
{\bf Noise Model} & {\bf Entry-wise Uncertainty} \\ 
\midrule 
Gaussian   & $\sigma^2(u^{* 2}_i + v^{* 2}_j)/p$ \vspace{3pt}\\
%\midrule
Poisson  & $M^{*}_{ij}(u^{*}_i \|u^{*}\|^3_3 + v^{*}_j \|v^{*}\|^3_3)/p$ \vspace{3pt} \\
%\midrule
Binary  & $M^{*}_{ij}(u^{*}_i \|u^{*}\|^3_3 + v^{*}_j \|v^{*}\|^3_3 - M_{ij}^{*}\|u^{*}\|^4_4 - M_{ij}^{*}\|v^{*}\|_4^4)/p$ \\  
 \bottomrule 
 \end{tabular}
 %}
 \caption{A comparison of uncertainty formulas for different noise models when $r=1$, i.e., $M^{*} = \sigma_{1} u^{*} v^{*\top}.$ See details in \cref{sec:results}. %Summary of results on synthetic data. AUC, $\|\hat{M}-M\|_{\F}$, and $\|\hat{M}-M\|_{\max}$ are averaged over 1000 instances. Evaluated algorithms include an ideal algorithm that knows $M^*$ and the anomaly model ($\pi^*$), our algorithm (EW), and three existing benchmarks.
}
 \label{tb:comparison-model}
%\scriptsize{F refers to $|\hat{M}-M|_{\text{F}}$; Max refers to $|\hat{M}-M|_{\text{max}}$.}
\end{center}
%\vspace{-6mm}
\end{table*}

Consider the problem of {\em low-rank matrix completion}: there exists a low-rank matrix that we seek to recover, having observed only a subset of its entries, each perturbed by additive noise. A rich stream of research over the past two decades has essentially solved this problem -- there exist efficient algorithms  which achieve order-optimal recovery guarantees under provably-minimal assumptions \citep{candes2009exact,candes2010matrix,keshavan2010matrix}. Further advances have yielded (and continue to yield) algorithmic improvements \citep{mazumder2010spectral,jain2013low,tanner2016low,dong2021riemannian}, and a deeper understanding of the optimization landscape itself \citep{ge2016matrix,zhu2017global}. 

Naturally, these algorithms have been applied in a vast array of applications, including recommendation systems, bioinformatics, network localization, and modern commerce \citep{su2009survey,natarajan2014inductive,so2007theory,amjad2017censored}, just to name a few. Now many of these applications require, in addition to scalability and accuracy, the ability to quantify the {\em uncertainty} of an estimator -- for example, something as seemingly-straightforward as confidence intervals on the estimated entries of a matrix.

Such an {\em uncertainty quantification} procedure, analogous to existing procedures for problems like linear regression, would ideally (a) apply to a commonly-used estimator, (b) require no more additional computation than the estimator itself, and (c) be justified by a (limiting) distributional characterization. Given the volume and success of the research just described, it is perhaps surprising that this problem has been largely unsolved (see the Related Work for past progress).

%In contrast, the study of the inference and uncertainty quantification of recovered entries remains a long standing question, despite highly desired in practice. 
Fortunately, there was a recent ``breakthrough.'' Applying newer techniques such as the \textit{leave-one-out} technique and fine-grained entry-wise analysis \citep{ma2018implicit,ding2020leave,abbe2020entrywise},  \cite{CFMY:19,chen2020noisy} proposed an uncertainty quantification technique for matrix completion, which satisfies the three ``ideal'' conditions above, in the case of {\em homogeneous Gaussian noise}. Further progress in \cite{xia2021statistical} extended this to homogeneous sub-Gaussian noise.

{\bf Toward ``Realistic'' Noise:} A gap still exists when we seek to apply these inferential results in practice, since many applications have more sophisticated noise models (namely, heterogeneous and sub-exponential noise). For example, in discrete panel sales data, the observation for sales at a location during a period of time is commonly modeled as {\em Poisson} with a certain expected sales rate \citep{amjad2017censored,shi2014production}. Similarly in web-commerce systems, data indicating clicks or purchases is often binary and modeled as {\em Bernoulli} random variables \citep{ansari2003customization,grover1987simultaneous}. 

Thus motivated, in this work we establish the first uncertainty quantification results for matrix completion with {\em heterogenous and sub-exponential noise}. Precisely, we characterize the distribution of recovered matrix entries from common estimators. An application of our results can already be seen in \cref{tb:comparison-model}, where we have derived explicit formulas under Poisson and Binary noise, which are distinctive from the homogeneous Gaussian noise case already existing in the literature. In addition, we demonstrate the quality of our procedure through experiments on real sales data. 
The proof of our main result generalizes the proof framework in \citep{CFMY:19}, leveraging recent results for sub-exponential matrix completion from \cite{mcrae2019low}, and a new high-dimension concentration bound (\cref{lem:subexponential-Bernstein}), which may be of independent interest. 
  
%In the context of sub-exponential matrix completion, \cite{mcrae2019low} established the Frobenius error guarantees $\|\hat{M}-M^{*}\|_{\F}$, \cite{farias2021near} showed the bounds on entry-wise error $\|\hat{M}-M^{*}\|_{\max}$. This work makes one step further by showing that $\hat{M}_{ij}-M^{*}_{ij} \sim \mathcal{N}(0, s_{ij}^2)$ where $s_{ij}$ is defined in \cref{thm:main-theorem}. 

\textbf{Related Work:} This paper is related to at least three streams of work. The first is, naturally, {\em uncertainty quantification in matrix completion}. Besides the works described above, prior approaches to this were based on either (a) converting recovery guarantees on matrix norms to confidence regions \citep{carpentier2015uncertainty,carpentier2018adaptive}, (b) the Bayesian formulation of matrix completion \citep{salakhutdinov2008bayesian,fazayeli2014uncertainty,tanaka2021bayesian,alquier2015bayesian}, or (c) deep-learning-bsaed methods \citep{lakshminarayanan2016simple,zeldes2017deep}.
The second stream relates to {\em sub-exponential matrix completion}. \cite{mcrae2019low} established guarantees on the Frobenius error $\|\hat{M}-M^{*}\|_{\F}$; \cite{farias2021near} established entry-wise error guarantees. This work makes takes one step further with an entry-wise distributional characterization of the error. Finally, there is a line of work, in multi-variate linear regression, advocating the use of {\em heteroskedasticity-robust variance estimators} instead of homoskedasticity estimators, since the former are more robust to heterogeneous noise \citep{long2000using,hayes2007using,imbens2016robust,cribari2014new}. Our work is in the same spirit, but in the context of matrix completion.

\paragraph{Notation:} 
The sub-exponential norm of a random variable $X$ is defined as $ \norm{X}_{\psi_1} := \inf\{t > 0: \E{\exp(|X|/t)}\leq 2\}.$ For a matrix $A \in \R^{m\times n}$, we abbreviate $\sum_{(i,j)\in [m]\times [n]} A_{ij}$ as $\sum_{ij} A_{ij}$ when no ambiguity exists. We require a few matrix norms: $\norm{A}_{2,\infty}^2 := \max_{i}\sum_{j} A_{ij}^2$, $\norm{A}_{\max} = \max_{ij} |A_{ij}|$, and $\norm{A}_{\F}^2 = \sum_{ij} A_{ij}^2$. The spectral norm is denoted $\norm{A}_2.$

\section{Model} \label{sec:model}
Let $M^{*} \in \mathbb{R}^{m\times n}$ be a rank-$r$ matrix, where $m \le n$ without loss of generality. Let $O = M^{*} + E$ be the realization of $M^{*}$ corrupted by a noise matrix $E \in \R^{m\times n}$. We observe $P_{\Omega}(O)$, which is the subset of entries of $O$ restricted to an observation set $\Omega \subset [m] \times [n]$:
\begin{align*}
    P_{\Omega}(O)_{ij} = \begin{cases}O_{ij} & (i, j) \in \Omega \\ 0 & (i, j) \notin \Omega \end{cases}.
\end{align*}
The {\em matrix completion} problem is to recover $M^{*}$ from this noisy and partial observation $P_{\Omega}(O).$

Let $M^{*} = U^{*}\Sigma^{*}V^{*\top}$ be the SVD of $M^{*}$. Here, $\Sigma^{*} \in \R^{r\times r}$ is a diagonal matrix with singular values $\sigma_{\max} = \sigma_{1}^{*} \geq \sigma_{2}^{*} \geq \dotsc \geq \sigma_{r}^{*} = \sigma_{\min}$; and $U^{*} \in \R^{m \times r}, V^{*} \in \R^{n\times r}$ contain the left and right-singular vectors. Let $\kappa = \sigma_{\max}/\sigma_{\min}$ be the condition number of $M^{*}$. 

We will make three assumptions. The first two are, by this point, canonical in the matrix completion literature \citep{candes2010matrix,keshavan2010matrix,ma2018implicit,abbe2020entrywise,CFMY:19}:
\begin{assumption}[Uniform Sampling]\label{assum:random-sampling}
Each element of $[m] \times [n]$ is included in $\Omega$ independently, and with probability $p$. 
\end{assumption}

\begin{assumption}[Incoherence]\label{assum:incoherence}
\begin{align}
    \norm{U^{*}}_{2,\infty} \leq \sqrt{\frac{\mu r}{m}} \text{\quad and \quad} \norm{V^{*}}_{2,\infty} \leq \sqrt{\frac{\mu r}{n}} \label{eq:incoherence}
\end{align}
\end{assumption}

Finally, our third assumption is a {\em generalization} of the independent (and often homogeneous), sub-Gaussian noise that is typically assumed in the literature \citep{CFMY:19,xia2021statistical}. As described above, this generalization enables a host of practical applications, such as those arising in counting data and panel sales data \citep{amjad2017censored,ansari2003customization}. 

\begin{assumption}[Independent Sub-exponential Noise]\label{assum:noise}
The entries of $E$ are independent, mean-zero random variables with variances $\sigma_{ij}^2$, and are also independent from $\Omega$. Furthermore,  $\norm{E_{ij}}_{\psi_1} \leq L$ for every $(i, j)$, where $\|\cdot\|_{\psi_1}$ is the sub-exponential norm. 
\end{assumption}

\section{Algorithm}
In this section, we describe a ``de-biased'' estimator $M^{\dd}$ for $M^{*}$. This was originally proposed in \citep{CFMY:19}, where the uncertainty quantification for $M^{\dd}$ is characterized under homogeneous, Gaussian noise. Motivated by practical applications, we study new uncertainty quantification formulas for $M^{\dd}$ under heterogenous sub-exponential noise. 

%The study of this estimator enables the breakthrough of uncertainty quantification analyses in matrix completion with homogeneous Gaussian noise \citep{CFMY:19}. We 
To begin, consider a natural least-square estimator for $M^{*}$
\begin{align} \label{eq:original-optimization}
\hat{M} \triangleq \argmin_{M' \in \R^{m\times n}, \rank(M')=r} \frac{1}{2p} \norm{P_{\Omega}(O-M')}_{\F}^2
\end{align}
Here, $\hat{M}$ is the projection of $M$ into the set of rank-$r$ matrices in regard to Euclidean distance (restricted on the set $\Omega$).  

Directly solving \cref{eq:original-optimization} turns out to be a challenge task. A popular method is to represent $M' = XY^{\top}$ where $X \in \R^{m\times r}, Y \in \R^{n\times r}$ are low-rank factors, and solve the following non-convex regularized optimization problem
\begin{align}
    \minimize_{X \in R^{n_1\times r}, Y \in R^{n_2\times r}} f(X, Y) \label{eq:objective-function}
\end{align}
where 
\begin{align*}
f(X, Y) 
&:= \frac{1}{2p} \norm{P_{\Omega}(XY^{T} - O)}_{\F}^2 \\
&\quad + \frac{\lambda}{2p} \norm{X}_{\F}^2 + \frac{\lambda}{2p} \norm{Y}_{\F}^2. 
\end{align*}
With proper initializations, simple first-order methods are often sufficient to solve \cref{eq:objective-function} \citep{sun2016guaranteed}. The regularizer $\lambda > 0$ here is used to promote additional structure properties. For example, when gradient descent is performed, a positive $\lambda$ is critical for analyzing the convergence properties and also helps to achieve a balance between $X$ and $Y$ \citep{chen2020noisy}.  

However, the use of $\lambda$ also introduces additional bias to the estimator in \cref{eq:objective-function}, which has been a major obstacle to analyze the uncertainty quantification properties. \citep{CFMY:19} proposes a de-bias procedure to remove the bias brought by $\lambda$, based on the solution of \cref{eq:objective-function}. The algorithm is summarized below\footnote{We assume $\lambda \asymp L\log(n)\sqrt{np}, t^{\star} \asymp n^{23}, \eta \asymp 1/(n^{6}\kappa^3\sigma_{\max})$ throughout the paper, if not specified explicitly.}. 

\begin{algorithm}[H]
\caption{Gradient Descent with De-bias} \label{alg:GD}
{\bf Input:} $P_{\Omega}(O)$
\begin{algorithmic}[1]

\State{\textbf{{Spectral initialization}}: $X^{0} = U\sqrt{\Sigma}, Y^{0}=V\sqrt{\Sigma}$ where  $U\Sigma V^{\top}$ is the top-$r$ partial SVD decomposition of $\frac{1}{p}P_{\Omega}(O)$.} 

\State{\textbf{{Gradient updates}}: \textbf{for }$t=0,1,\ldots,t_{\star}-1$
\textbf{do}
 \begin{subequations}\label{subeq:gradient_update_ncvx-loo}
\begin{align*}
X^{t+1}= & X^{t}-\frac{\eta}{p}[P_{\Omega}(X^{t}Y^{t\top}-O)Y^{t}+\lambda X^{t}];\\
Y^{t+1}= & Y^{t}-\frac{\eta}{p}[P_{\Omega}(X^{t}Y^{t\top}-O)^{T}X^{t}+\lambda Y^{t}]
\end{align*}
\end{subequations}
where $\eta$ determines the learning rate. 
}

\State{\textbf{De-bias}: 
\begin{align}
    X^{\dd} &= X^{t_{\star}} \left(I_r + \frac{\lambda}{p}\left(X^{t_{\star} \top}X^{t_{\star}}\right)^{-1}\right)^{1/2} \label{eq:Xdebias}\\
    Y^{\dd} &= Y^{t_{\star}} \left(I_r + \frac{\lambda}{p}\left(Y^{t_{\star}\top}Y^{t_{\star}}\right)^{-1}\right)^{1/2}\label{eq:Ydebias}
\end{align}
}
\end{algorithmic}
{\bf Output:} $M^{\dd} = X^{\dd}Y^{\dd \top}$
\end{algorithm}
Steps 1 and 2 in \cref{alg:GD} form a typical gradient descent procedure for solving \cref{eq:objective-function}. 
The de-biasing step, i.e. \cref{eq:Xdebias,eq:Ydebias} in \cref{alg:GD}, is critical for enabling the uncertainty quantification analysis.

We will use the remainder of this section (which can be skipped without loss of continuity) to provide some intuition for the peculiar form of \cref{eq:Xdebias,eq:Ydebias} based on first-order conditions.  Consider an example with $p=1$ (no entry is missing). Since $O$ is fully observed, let $O=U_{r}\Sigma_{r}V_{r}^{\top} + U_{n-r}\Sigma_{n-r}V^{\top}_{n-r}$ be the SVD of $O$, where $\Sigma_{r}$ corresponds to the largest $r$ singular values and $\Sigma_{n-r}$ corresponds to the remaining one. Then it follows that the optimal solution of  \cref{eq:original-optimization} is $\hat{M} = U_{r} \Sigma_{r} V_{r}^{\top}$ \citep{eckart1936approximation}.

Next, consider the regularized objective \cref{eq:objective-function}. We can derive that the optimal solution $(X, Y)$ for \cref{eq:objective-function} has the form 
\begin{align*}
X = U_{r}(\Sigma_{r}-\lambda I_{r})^{1/2}, \quad Y = V_{r}(\Sigma_{r} - \lambda I_{r})^{1/2}.
\end{align*}
In fact, this can be verified from the first-order conditions, 
\begin{align*}
\frac{\partial f(X, Y)}{\partial X} 
&= (XY^{\top} - O)Y + \lambda X\\
&= (U_{r}(\Sigma_{r}-\lambda I)V_{r}^{\top} - O)Y + \lambda X\\
&= (U_{n-r}\Sigma_{n-r}V_{n-r}^{\top} - \lambda U_{r}V_{r}^{\top}) Y + \lambda X\\
&\overset{(i)}{=} -\lambda U_{r}V_{r}^{\top} V_{r} (\Sigma_{r}-\lambda I_{r})^{1/2} + \lambda X\\
&\overset{(ii)}{=}0,
\end{align*}
where in (i) we use that $V_{n-r}^{\top}Y = V_{n-r}^{\top}V_{r}(\Sigma_{r}-\lambda I_{r})^{1/2} = 0$, and in (ii) we use that $V_{r}^{\top}V_{r} = I_{r}.$ Similarly $\frac{\partial f(X, Y)}{\partial Y}=0$ also vanishes. 

Then, this justifies the particular de-biased form in \cref{eq:Xdebias,eq:Ydebias}: 
\begin{align*}
X^{\dd} 
&= X (I_{r} + \lambda (\Sigma_{r} - \lambda I_{r})^{-1})^{1/2}\\
&= X(\Sigma_{r} (\Sigma_{r} - \lambda I_{r})^{-1})^{1/2}\\
&= U_{r} (\Sigma_{r} - \lambda I_{r})^{1/2}  (\Sigma_{r} - \lambda I_{r})^{-1/2} \Sigma_{r}^{1/2}\\
&= U_{r} \Sigma_{r}^{1/2}.
\end{align*}
Similarly, $Y^{\dd} = V_{r}\Sigma_{r}^{1/2}.$ Thus $X^{\dd}Y^{\dd \top} = U_{r}\Sigma_{r}V_{r}^{\top}$ is the desired optimal solution of \cref{eq:original-optimization}.

%In fact, this step can be viewed as extracting $M^{\dd}$,  an approximated solution of \cref{eq:original-optimization}, from the approximated solution of \cref{eq:objective-function}. To develope an intuition, consider the example that $r=1$ and $x\in R^{n\times 1}, y \in R^{n\times 1}$ is the optimal solution of \cref{eq:objective-function} with $\norm{x}=\norm{y}=D$. The first order condition of \cref{eq:objective-function} implies
%\begin{align}
%P_{\Omega}(xy^{T} - M) = \frac{-\lambda}{D} xy^{T}
%\end{align} 

% The initialization is done by the SVD decomposition. In particular, let 
% $$U\Sigma V^{T} = SVD(M')_{r}, U^{*}\Sigma^{*}V^{*T} = SVD(M^{*})_{r}$$ where $M'$ is obtained from $M$ by filling all missing entries with 0. Let $X^{0} = U\sqrt{\Sigma}, Y^{0} = V\sqrt{\Sigma}, X^{*} = U^{*}\sqrt{\Sigma^{*}}, Y^{*}=V^{*}\sqrt{\Sigma^{*}}.$

% Then, in the $t$-th step, (let $\eta$ be the step size), 
% \begin{align}
%     X^{t+1} &= X^{t} - \eta \nabla f(X^{t}, Y^{t}) = X^{t} - \frac{\eta}{p} (P_{\Omega}(X^{t}Y^{tT} - M) Y^{t} + \lambda X^{t})\\
%     Y^{t+1} &= Y^{t} - \eta \nabla f(X^{t}, Y^{t}) = Y^{t} - \frac{\eta}{p} (P_{\Omega}(X^{t}Y^{tT} - M)^{T} X^{t} + \lambda Y^{t}). 
% \end{align}

% By choosing the proper $T = n^{18}, \eta = \frac{1}{n^2}$, we obtain $X = X^{T}, Y = Y^{T}.$

% Let 
% $$
% X^{d} = X (I_r + \frac{\lambda}{p} (X^{T}X)^{-1})^{1/2}, Y^{d} = Y(I_r + \frac{\lambda}{p} (Y^{T}Y)^{-1})^{1/2}.
% $$

% We will use $M^{d} = X^{d}Y^{dT}$ as an estimator for $M^{*}.$

\section{Results}\label{sec:results}
We can now state our main result: an uncertainty quantification for $M^{\dd}$ under heterogeneous, sub-exponential noise. 
\begin{theorem}\label{thm:main-theorem}
Assume $m p \gg \kappa^4 \mu^2 r^2 \log^3 n$ and $L\log(n)\sqrt{n/p} \ll \sigma_{\min}/\sqrt{\kappa^4 \mu r\log n}$. Then for every $(i,j) \in [m]\times [n]$, we have 
\begin{align*}
    &\sup_{t\in \R}\left|P\left\{\frac{M^{\dd}_{ij} - M^{*}_{ij}}{s_{ij}} \leq t\right\} - \Phi(t)\right| \lesssim s_{ij}^{-3} \frac{L^2 \mu^3 r^{3}}{m^2p}+ \\
    &\quad s_{ij}^{-1} \left(\frac{L^2\log^{3}(n)\mu r \kappa^5}{p\sigma_{\min}} + \frac{L\mu^2 r^2 \log^{2}(n)\kappa^4}{p m} \right)+ \frac{1}{m^{10}}, 
\end{align*}
where $\Phi(\cdot)$ is the CDF of the standard Gaussian, and $s_{ij}>0$ is given by 
\begin{align}
    s_{ij}^2 := \frac{\sum\limits_{l=1}^{m} \sigma_{lj}^2 \left(\sum\limits_{k=1}^{r} U^{*}_{ik}U^{*}_{lk}\right)^2 + \sum\limits_{l=1}^{n}\sigma_{il}^2  \left(\sum\limits_{k=1}^{r} V^{*}_{lk} V^{*}_{jk}\right)^2}{p}. \label{eq:variance-def}
\end{align}
\end{theorem}
To quickly parse this result, note that a typical scaling of the parameters would see $m=\Theta(n)$, $np \gtrsim \log^{6}(n)$, $\mu=r=\kappa=L=O(1)$, $\sigma_{\min}=\Omega(n)$, $\sigma_{ij}=\Omega(1)$, and $\|V^{*}_{j,\cdot}\|=\|U^{*}_{i,\cdot}\|=\Omega(\sqrt{1/n})$. \cref{thm:main-theorem} would then imply that 
\begin{align}
\frac{M^{\dd}_{ij} - M^{*}_{ij}}{s_{ij}} \overset{d}{\longrightarrow} \mathcal{N}(0, 1) \label{eq:Gaussian-approximation}
\end{align}
where $s_{ij}$ is defined in \cref{eq:variance-def}. This is precisely the type of characterization we sought at the outset.
The form of $s_{ij}$, as defined in \cref{eq:variance-def}, is of course critical to the characterization, and probably best understood via a few examples:

{\em 1. Homogeneous Gaussian Noise.} First as a sanity check, when $E_{ij} \sim \mathcal{N}(0, \sigma^2)$, \cref{thm:main-theorem} reduces to the same variance formula as Theorem 2 in \citep{CFMY:19}:
\begin{align}
s_{ij}^2 = \frac{\sigma^2 (\|U_{i,\cdot}^{*}\|^2 + \|V_{j,\cdot}^{*}\|^2)}{p}. \label{eq:homogeneous-formula}
\end{align}

{\em 2. Poisson Noise.} When the observations are Poisson, i.e. $O_{ij} \sim \text{Poisson}(M_{ij}^{*})$, the variance of the noise $E_{ij}$ is $\sigma_{ij}^2 = \text{Var}(O_{ij}-M^{*}_{ij}) = M^{*}_{ij}$. Then applying  \cref{thm:main-theorem}, we have that $M^{\dd}_{ij} - M^{*}_{ij} \sim \mathcal{N}(0, s_{ij}^2)$ where  
\begin{align}
 s_{ij}^2 = \frac{\sum\limits_{l=1}^{m} M_{lj}^{*} \left(\sum\limits_{k=1}^{r} U^{*}_{ik}U^{*}_{lk}\right)^2 + \sum\limits_{l=1}^{n}M_{il}^{*}  \left(\sum\limits_{k=1}^{r} V^{*}_{lk} V^{*}_{jk}\right)^2}{p}. \label{eq:Poisson-noise}
\end{align}
A special case is when $r=1$ and $M^{*} = \sigma_{1} u^{*}v^{*\top}$, for which we have
\begin{align*}
s_{ij}^2 
&= \frac{\sum_{l=1}^{m} M_{lj}^{*} \left(u^{*}_{l}u^{*}_{i}\right)^2 + \sum_{l=1}^{n}M_{il}^{*}  \left(v^{*}_{l} v^{*}_{j}\right)^2}{p}\\
&= \frac{\sum_{l=1}^{m} \sigma_1 u_{l}^{*} v_{j}^{*} \left(u^{*}_{l}u^{*}_{i}\right)^2 + \sum_{l=1}^{n} \sigma_1 u_{i}^{*} v_{l}^{*} \left(v^{*}_{l} v^{*}_{j}\right)^2}{p}\\
&= \frac{\sigma_1 v_{j}^{*} u^{*2}_i \sum_{l=1}^{m} u_l^{*3} + \sigma_1 u_i^{*} v_j^{*2} \sum_{l=1}^{n} v^{*3}_{l}}{p}\\
&= \frac{M^{*}_{ij} (u^{*}_{i} \|u^{*}\|_3^3 + v^{*}_{j} \|v^{*}\|_3^3)}{p},
\end{align*}
which corresponds to the formula in \cref{tb:comparison-model}.

{\em 3. Binary Noise.} Finally, binary observations occur frequently in applications. For example, in a recommender system or e-commerce platform, $O_{ij} \in \{0,1\}$ can represent whether the $i$th user viewed (or purchased) the $j$th item (or product) \citep{ansari2003customization,grover1987simultaneous,farias2019learning}. A common noise model for such observations is to assume the $O_{ij}$ are Bernoulli random variables with mean $M^{*}_{ij}$, i.e., $O_{ij} \sim \text{Ber}(M^{*}_{ij})$.  

With such binary observations, the variance of the noise $E_{ij}$ is $\sigma_{ij}^2 = \text{Var}(O_{ij}-M^{*}_{ij}) = M_{ij}^{*}(1-M_{ij}^{*}).$  %from \cref{eq:Gaussian-approximation}, $M^{\dd}_{ij} - M^{*}_{ij} \sim \mathcal{N}(0, s_{ij}^2)$ where  
Then $s_{ij}$ takes the form
\begin{align*}
 s_{ij}^2 &= \frac{\sum_{l=1}^{m} M_{lj}^{*}(1-M^{*}_{lj}) \left(\sum_{k=1}^{r} U^{*}_{ik}U^{*}_{lk}\right)^2}{p} \\
 &\quad +\frac{\sum_{l=1}^{n}M_{il}^{*}(1-M^{*}_{il})  \left(\sum_{k=1}^{r} V^{*}_{lk} V^{*}_{jk}\right)^2}{p}.
\end{align*}
When $r=1$ and $M^{*} = \sigma_{1} u^{*}v^{*\top}$, we have
\begin{align*}
s_{ij}^2 
%&= \frac{\sum\limits_{l=1}^{m} M_{lj}^{*}(1-M_{lj}^{*}) \left(u^{*}_{l}u^{*}_{i}\right)^2 + \sum\limits_{l=1}^{n}M_{il}^{*}(1-M_{il}^{*})  \left(v^{*}_{l} v^{*}_{j}\right)^2}{p}\\
&= \frac{\sum_{l=1}^{m} \sigma_1 u_{l}^{*} v_{j}^{*} (1-\sigma_1 u_{l}^{*} v_{j}^{*}) \left(u^{*}_{l}u^{*}_{i}\right)^2}{p} \\
&\quad + \frac{\sum_{l=1}^{n} \sigma_1 u_{i}^{*} v_{l}^{*} (1-\sigma_1 u_{i}^{*} v_{l}^{*})\left(v^{*}_{l} v^{*}_{j}\right)^2}{p}\\
%&= \frac{\sigma_1 v_{j}^{*} u^{*2}_i \sum_{l=1}^{m} u_l^{*3} + \sigma_1 u_i^{*} v_j^{*2} \sum_{l=1}^{n} v^{*3}_{l} - \sigma_1^2 u_{i}^{*2}v_{j}^{*2} \sum_{l=1}^{m} u_l^{*4} - \sigma_1^{2} u_i^{*2}v_{j}^{*2} \sum_{l=1}^{n} v_{l}^{*4}}{p}\\
&= \frac{M^{*}_{ij} (u^{*}_{i} \|u^{*}\|_3^3 + v^{*}_{j} \|v^{*}\|_3^3 - M_{ij}^{*}(\|u^{*}\|_4^{4} + \|v^{*}\|_4^{4}))}{p}.
\end{align*}
%which verifies the formula in \cref{tb:comparison-model}.

\textbf{Empirical Inference:} In practice, the underlying $U^{*}$ and $V^{*}$ are not known, and thus $s_{ij}$ cannot be computed exactly. We propose the use of the corresponding empirical estimators to estimate $s_{ij}$ for the purposes of inference. Let $M^{\dd} = U^{\dd} \Sigma^{\dd} V^{\dd \top}$ be the SVD of $M^{\dd}$. For example, in the Poisson noise scenario, we would use the following empirical estimator for $s_{ij}$:
\begin{align*}
\hat{s}_{ij}^2 = \frac{\sum\limits_{l=1}^{m} M_{lj}^{\dd} \left(\sum\limits_{k=1}^{r} U^{\dd}_{ik}U^{\dd}_{lk}\right)^2 + \sum\limits_{l=1}^{n}M_{il}^{\dd}  \left(\sum\limits_{k=1}^{r} V^{\dd}_{lk} V^{\dd}_{jk}\right)^2}{p}.
\end{align*} 

In cases where $\sigma_{kl}$ is also unknown, we let $\hat{E}_{ij} = O_{ij} - M^{\dd}_{ij}$ be the empirical estimator for the noise. This procedure (i.e. the use of empirical estimators) can be justified via the following result: 
\begin{corollary}\label{cor:empirical-inference}
Follow the settings in \cref{thm:main-theorem}. Assume that $\forall (i,l), \sigma_{il}=\Theta(L)$ and 
$$
s_{ij} \gtrsim L^{2}\mu^2r^2\kappa^5\log^{4}(n)\left(\frac{1}{\sigma_{\min}p} + \frac{1}{mp} + \frac{1}{m^{2/3}p^{1/3}}\right).
$$
Let 
\begin{align*}
\hat{s}_{ij}^2 
&= \frac{\sum\limits_{l=1, (l,j) \in \Omega}^{m} \frac{1}{p}\hat{E}_{lj}^2 \left( \sum_{k=1}^{r} U^{\dd}_{ik} U^{\dd}_{lk}\right)^2}{p} \\
&\quad + \frac{\sum\limits_{l=1, (i,l) \in \Omega}^{n} \frac{1}{p}\hat{E}_{il}^2 \left( \sum_{k=1}^{r} V^{\dd}_{lk} V^{\dd}_{jk}\right)^2}{p}
\end{align*}
be the empirical estimator of $s_{ij}.$ Then under the same assumptions made in \cref{thm:main-theorem}, we have that
\begin{align*}
\sup_{t\in \R} \left| P\left\{\frac{M^{\dd}_{ij} - M^{*}_{ij}}{\hat{s}_{ij}} \leq t\right\} - \Phi(t) \right| = o(1).
\end{align*}
\end{corollary}
Additional justification for this procedure is given as experiments later on.

\textbf{Aside: When $s_{ij} \approx 0.$} Curious readers may note that $s_{ij}$ may be too small for \cref{thm:main-theorem} and \cref{cor:empirical-inference} to apply. In this case, although the Gaussian approximation in \cref{thm:main-theorem} does not hold, an entry-wise error bound still holds, and may be sufficient for many applications (see the Appendix for details): 
\begin{align*}
|M_{ij}^{\dd} - M^{*}_{ij}| \lesssim \kappa \mu r L \sqrt{\frac{\log(n)}{mp}}.
\end{align*}
An uncertainty characterization when $s_{ij} \approx 0$ involves a second-order error analysis and remains an open question. 

%
%\begin{theorem}\label{thm:main-theorem}
%\begin{align}
%    M^{d} - M^{*} = \frac{1}{p}P_{\Omega}(E)V^{*}V^{*T} + U^{*}U^{*T}\frac{1}{p}P_{\Omega}(E) + T
%\end{align}
%where $|T_{ij}| \lesssim \frac{\log n}{n}$ for every $(i, j)$.
%Furthermore,
%\begin{align}
%    M^{d}_{ij} - M^{*}_{ij} =  \left(\frac{1}{p}P_{\Omega}(E)^{(1)}V^{*}V^{*T}\right)_{ij} + \left(U^{*}U^{*T}\frac{1}{p}P_{\Omega}(E)^{(2)}\right)_{ij} + d_{ij}
%\end{align}
%where $\frac{1}{p}P_{\Omega}(E)^{(1)}, \frac{1}{p}P_{\Omega}(E)^{(2)}$ are independently drawn from the distribution on $E$, and $|d_{ij}| \lesssim \frac{\log n}{n}.$
%\begin{align}
%    Var(M^{d}_{ij} - M^{*}_{ij}) \approx \frac{\sum_{k,t} \sigma_{ik}^{2}(V_{kt}^{*})^2(V_{jt}^{*})^2 + \sum_{k,t}(U_{it}^{*})^2(U_{kt}^{*})^2\sigma_{kj}^{2}}{p}.
%\end{align}
%\end{theorem}
%When $r=1$, $Var(M^{d}_{ij} - M^{*}_{ij}) = M_{ij}^{*}(u_i\norm{u}_3^{3} + v_j\norm{v}_3^{3}).$

\section{Proof Overview}
In this section, we present the proof framework of \cref{thm:main-theorem} (see details in \cref{sec:append-proof-main-theorem}). In order to extend to heterogeneous sub-exponential noise from homogeneous Gaussian, we generalize the proof of \citep{CFMY:19} with the help of recent sub-exponential matrix completion results \citep{mcrae2019low} and a sub-exponential variant of matrix Bernstein inequality (\cref{lem:subexponential-Bernstein}).

Similar to \cite{CFMY:19}, our proof is based on the leave-one-out technique that has been recently used for providing breakthrough bounds for entry-wise analysis in matrix completion problems (dated back to \cite{ma2018implicit}, also see \cite{ding2020leave,abbe2020entrywise,chen2020noisy}). 

We establish the following key results to characterize the decomposition of low-rank factors $(X^{\dd}, Y^{\dd})$, as a heterogeneous sub-exponential generalization of Theorem 5 in \cite{CFMY:19}. 
\begin{theorem}\label{thm:eigen-vector}
Assume $m p \gg \kappa^4 \mu^2 r^2 \log^3 n$ and $L\log(n)\sqrt{\frac{n}{p}} \ll \frac{\sigma_{\min}}{\sqrt{\kappa^4 \mu r\log n}}$. There exists a rotation matrix $H^{\dd} \in \O^{r\times r}$ and $\Phi_X \in \R^{m\times r}, \Phi_Y \in \R^{n\times r}$ such that the following holds with probability $1-O(n^{-10})$,
\begin{align*}
    X^{\dd}H^{\dd} - X^{*} &= \frac{1}{p}P_{\Omega}(E)Y^{*}(Y^{*T}Y^{*})^{-1} + \Phi_{X}\\
    Y^{\dd}H^{\dd} - Y^{*} &= \frac{1}{p}P_{\Omega}(E)^{T}X^{*}(X^{*T}X^{*})^{-1} + \Phi_{Y}
\end{align*}
where 
\begin{align*}
&\max\left\{\norm{\Phi_X}_{2,\infty}, \norm{\Phi_Y}_{2,\infty}\right\} 
\lesssim \\
&\frac{L \log n}{\sqrt{p\sigma_{\min}}}\left(  \frac{L \log n }{\sigma_{\min}} \sqrt{\frac{\kappa^9 \mu r n\log n}{p}} + \sqrt{\frac{\kappa^{7}\mu^3 r^{3}\log^2 n}{mp}}\right).
\end{align*}
\end{theorem}

\begin{proof}
At a high level, the proof of \cref{thm:eigen-vector} follows a similar proof of Theorem 5 in \citep{CFMY:19}, but with replacements that employ more fine-grained analyses of $E$ for whenever the Gaussianity of $E$ is used in \citep{CFMY:19}. These analyses aim to address the sub-exponentiality and heterogeneity of $E$, with the help of the following two lemmas.

\begin{lemma}\label{lem:subexponential-Bernstein}
Given $k$ independent random $m_1 \times m_2$ matrices $X_1, X_2, \dotsc, X_{k}$ with $E[X_i] = 0.$ Let 
\begin{align*}
    V := \max\left(\norm{\sum_{i=1}^{k} E[X_i X_i^{T}]}, \norm{\sum_{i=1}^{k} E[X_i^{T} X_i]}\right).
\end{align*}
Suppose $\norm{\norm{X_i}}_{\psi_1} \leq B$ for $i \in [k].$ Then, 
\begin{align*}
    &\norm{X_1+X_2+\dotsc + X_{k}}\lesssim  \\ 
    &\sqrt{V\log(k(m_1+m_2))} + B\log(k(m_1+m_2))\log(k)
\end{align*}
with probability $1-O(k^{-c})$ for any constant $c$.
\end{lemma}

\begin{lemma}\label{lem:operator-norm-bound}
Suppose $E \in R^{m\times n}$ ($m\leq n$) whose entries are independent and centered. Suppose $\norm{E_{ij}}_{\psi_1} \leq L$ for any $(i, j) \in [m] \times [n].$ Let $\Omega \in [m] \times [n]$ be the subset of indices where each index $(i, j)$ is included in $\Omega$ independently with probability $p$. Suppose $n p \geq c_0 \log^3 n$ for some sufficient large constant $c_0$. then, with probability $1-O(n^{-11})$,
\begin{align*}
    \norm{\frac{1}{p} P_{\Omega}(E)} \leq C L\sqrt{\frac{n}{p}}.
\end{align*}
\end{lemma}

Here, \cref{lem:subexponential-Bernstein} is a generalization of matrix Bernstein inequality in Theorem 6.1.1 of \citep{tropp2015introduction}. \cref{lem:operator-norm-bound} is an implication of Lemma 4 in \citep{mcrae2019low}. 

Equipped with \cref{lem:subexponential-Bernstein,lem:operator-norm-bound}, the desired bounds for sub-exponential $E$ can be established.  Following we provide an example of using \cref{lem:subexponential-Bernstein} to bound $\|X^{*T} P_{\Omega}(E) Y^{*}\|$, which is critical for obtaining the bounds in \cref{thm:eigen-vector}.

To begin, note that 
\begin{align*}
    X^{*T} P_{\Omega}(E) Y^{*} = \sum_{k=1}^{m} \sum_{l=1}^{n} X^{*}_{k,\cdot} Y_{l,\cdot}^{*\top} \delta_{k,l} E_{k, l}
\end{align*}
where $\delta_{k,l} \sim \Ber(p)$ indicates whether $(k, l) \in \Omega.$ Let $A_{k,l} := X^{*}_{k,\cdot} Y^{*\top}_{l,\cdot} \delta_{k,l} E_{k, l}$ for $k\in[m], l\in[n]$. Then,
\begin{align*}
\norm{X^{*T} P_{\Omega}(E) Y^{*}} = \norm{\sum_{k,l} A_{k,l}}.
\end{align*}
Note that $A_{k,l} \in \R^{r\times r}$ are independent zero-mean random matrices and we aim to invoke \cref{lem:subexponential-Bernstein} to bound $ \|\sum_{k,l} A_{k,l}\|.$ Let
\begin{align*}
V &:= \max\left(\norm{\sum_{k,l} E[A_{k,l}A_{k,l}^{\top}]}, \norm{\sum_{k,l} E[A_{k,l}^{\top}A_{k,l}]}\right) \\
B &:= \max_{k,l}\norm{\norm{A_{k,l}}}_{\psi_1}.
\end{align*}
Note that
\begin{align*}
\norm{\sum_{k,l} E[A_{k,l}A_{k,l}^{\top}]} 
&\leq \sum_{k,l} \norm{E[A_{k,l}A_{k,l}^{\top}]}\\
&= \sum_{k,l} \sigma_{k,l}^2 p \|X_{k,\cdot}^{*}\|^2 \|Y_{l,\cdot}^{*}\|^2\\
&\overset{(i)}{\leq} 2L^2 p \|X^{*}\|_{\F}^2 \|Y^{*}\|_{\F}^2
\end{align*}
where in (i) we use the fact that $E(x^2) \leq 2\|x\|^2_{\psi_1}$ for an sub-exponential zero-mean random variable $x$. Similarly, the bounds can be established for $ \|\sum_{k,l} E[A_{k,l}^{\top}A_{k,l}]\|.$ Hence $V \leq  2L^2 p \|X^{*}\|_{\F}^2 \|Y^{*}\|_{\F}^2.$ 

Then, consider
\begin{align*}
B &:= \max_{k,l}\norm{\norm{A_{k,l}}}_{\psi_1} \\
&\leq \max_{k,l} \norm{E_{k,l} \delta_{k,l}}_{\psi_1} \norm{X^{*}}_{2,\infty} \norm{Y^{*}}_{2,\infty}\\
&\overset{(i)}{\leq} L \norm{X^{*}}_{2,\infty} \norm{Y^{*}}_{2,\infty}
\end{align*}
where (i) we use that $\|E_{k,l} \delta_{k,l}\|_{\psi_1} \leq \|E_{k,l}\|_{\psi_1} \leq L.$ Then apply \cref{lem:subexponential-Bernstein}, with probability $1-O(n^{-11})$, we obtain the desired bound for $\|X^{*\top} P_{\Omega}(E) Y^{*}\|$
\begin{align*}
\norm{\sum_{k,l} A_{k,l}} 
&\lesssim \sqrt{V \log(n)} + B\log^{2}(n)\\
&\overset{(i)}{\lesssim} \sqrt{p}L \sigma_{\max} r \sqrt{\log(n)}  + \frac{\mu r}{n} L \sigma_{\max} \log^{2}(n)\\
&\overset{(ii)}{\lesssim} \sqrt{p}L \sigma_{\max} r \sqrt{\log(n)}
\end{align*}
where in (i) we use that $\|X^{*}\|_{\F}^2=\|Y^{*}\|_{\F}^2 \leq \sigma_{\max} r$ and the incoherence condition \cref{eq:incoherence}, in (ii) we use that $m p \gg \kappa^4 \mu^2 r^2 \log^3 n.$

To establish a similar bound for $\|X^{*\top} P_{\Omega}(E) Y^{*}\|$, \cite{CFMY:19} uses the Gaussianity of $E$, which is not applicable here. 

Similar to this example, we apply \cref{lem:subexponential-Bernstein} and \cref{lem:operator-norm-bound} with more fine grained analyses to address the sub-exponentiality and heterogeneity of $E$. See \cref{sec:eigen-vector} for full details.

\end{proof}

Note that the error of $M^{\dd}-M^{*}$ is closely related to the errors of low-rank factors $X^{\dd}H^{\dd}-X^{*}, Y^{\dd}H^{\dd} - Y^{*}$ through the following 
\begin{align}
M^{\dd} - M^{*}
&= X^{\dd}H^{\dd}H^{\dd \top} Y^{\dd \top} - X^{*}Y^{*\top}  \nonumber \\
&= (X^{\dd}H^{\dd} - X^{*})Y^{*\top} + X^{*}(Y^{\dd}H^{\dd} - Y^{*})^{\top} \nonumber \\
&\quad - (X^{\dd}H^{\dd} - X^{*})(Y^{\dd}H^{\dd} - Y^{*})^{\top} \nonumber\\
&\overset{(i)}{\approx} (X^{\dd}H^{\dd} - X^{*})Y^{*\top} + X^{*}(Y^{\dd}H^{\dd} - Y^{*})^{\top} \nonumber
\end{align}
where in (i) we ignore the second-order error term $(X^{\dd}H^{\dd} - X^{*})(Y^{\dd}H^{\dd} - Y^{*})^{\top}.$ Note that \cref{thm:eigen-vector} implies that
\begin{align*}
X^{\dd}H^{\dd} - X^{*} &\approx \frac{1}{p}P_{\Omega}(E)Y^{*}(Y^{*T}Y^{*})^{-1} \\
    Y^{\dd}H^{\dd} - Y^{*} &\approx \frac{1}{p}P_{\Omega}(E)^{\top}X^{*}(X^{*T}X^{*})^{-1}
\end{align*}
Plug this into the decomposition of $M^{\dd} - M^{*}$, we have
\begin{align*}
M^{\dd} - M^{*} 
&\approx \frac{1}{p}P_{\Omega}(E)Y^{*}(Y^{*T}Y^{*})^{-1}Y^{*\top} \\
&\quad + \frac{1}{p}X^{*}(X^{*T}X^{*})^{-1}X^{*\top}P_{\Omega}(E)\\
&= \frac{1}{p}P_{\Omega}(E)V^{*}V^{*\top} + \frac{1}{p}U^{*}U^{*\top} P_{\Omega}(E).
\end{align*}
The results of \cref{thm:main-theorem} then follow from the above approximation and the use of Berry-Esseen type of inequalities. See \cref{sec:proof-main-theorem} for full details. 

\section{Experiments}
We evaluate the results in \cref{thm:main-theorem} for synthetic data under multiple settings. We then compare the performances of various uncertainty quantification formulas in real data.

\textbf{Synthetic Data.} We generate an ensemble of instances. Each instance consists of a few parameters: (i) $(m, n)$: the size of $M^{*}$; (ii) $r$: the rank of $M^{*}$; (iii) $p$: the probability of an entry being observed; (iv) $\bar{M}^{*}:$ the entry-wise mean of $M^{*}$ ($\bar{M}^{*} = \frac{1}{m n}\sum_{ij} M^{*}_{ij}$).

Given $(m, n, r, p, \bar{M}^{*})$, we follow the typical procedures of generating random non-negative low-rank matrices in \citep{cemgil2008bayesian,farias2021near}. Each instance is generated in two steps: (i) Generate $M^{*}$: let $U^{*} \in \R^{m\times r}, V^{*} \in \R^{n \times r}$ be random matrices with independent entries from $\text{Gamma}(2,1)$. Set $M^{*} = k U^{*}V^{*\top}$ where $k \in \R$ is picked such that $\frac{1}{m n}\sum_{ij} M^{*}_{ij} = \bar{M}^{*}.$ (ii) Generate $P_{\Omega}(O)$: then $O_{ij} = \text{Poisson}(M^{*}_{ij})$ and entries in $\Omega$ is sampled independently with probability $p$.\footnote{Here, we focus on the results of Poisson noise, where the results under the binary noise are similar in the experiments.}  

We first verify the entry-wise distributional characterization $M^{\dd}_{ij} - M^{*}_{ij} \sim \mathcal{N}(0, s_{ij}^2)$ where $s_{ij}$ is specified in \cref{eq:variance-def}. See a demonstration of the Gaussian approximality of the empirical distribution $(M^{\dd}_{ij} - M^{*}_{ij})/s_{ij}$ in \cref{fig:distribution}. Given an instance, we compute the coverage rate (the percentage of coverage of entries) that correponds to the 95\% confidence interval, where an ``coverage'' of an entry $(i, j)$ occurs if 
\begin{align*}
M^{\dd}_{ij} \in [M^{*}_{ij} - 1.96 s_{ij}, M^{*}_{ij} + 1.96 s_{ij}].
\end{align*}
The average coverage rates under different settings are shown in \cref{tb:coverage-rate}. The closeness of the results (ranging from $91\% - 95\%$) to the ``true'' coverage rate $95\%$ suggests the applicability of inference based on our variance formula. The trends in \cref{tb:coverage-rate} are also consistent with the intuition: the performance starts to degrade when $r$ increases, $p$ decreases, and the noise to signal ratio increases (decrease of $\bar{M}^{*}$).

\begin{figure}[h!]
\center
    \includegraphics[scale=0.4]{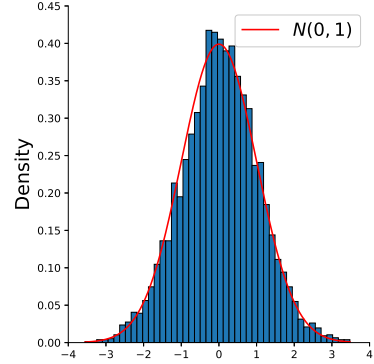}
    \caption{Empirical distribution of $(M^{\dd}_{11}-M^{*}_{11}) / s_{11}$ with $m=n=300, r=2, p=0.6$, and $\bar{M}^{*}=20.$ }
        \label{fig:distribution}
\end{figure}

\begin{table}[h!]%{0.41\textwidth}
%\hspace{-1mm}
%\vspace{2.5em}
\begin{center}
%\small{
\begin{tabular}{@{}cc@{}}
\toprule
$(r, p, \bar{M}^{*})$ & {\bf Coverage Rate} \\ 
\midrule 
(3, 0.3, ~5) & 0.936 ($\pm$ 0.003)\\
(3, 0.3, 20) & 0.945 ($\pm$ 0.004)\\
(3, 0.6, ~5) & 0.947 ($\pm$ 0.003)  \\
(3, 0.6, 20) & 0.949 ($\pm$ 0.003)\\
(6, 0.3, ~5) & 0.910 ($\pm$ 0.002)\\
(6, 0.3, 20) & 0.934 ($\pm$ 0.002)\\
(6, 0.6, ~5) & 0.934 ($\pm$ 0.003)\\
(6, 0.6, 20) & 0.943 ($\pm$ 0.003)\\
 \bottomrule 
 \end{tabular}
 %}
 \caption{Coverage rates for different $(r, p, \bar{M}^{*})$ with $m=n=500.$ The empirical mean and empirical standard deviation are reported over 100 instances.}
 \label{tb:coverage-rate}
%\scriptsize{F refers to $|\hat{M}-M|_{\text{F}}$; Max refers to $|\hat{M}-M|_{\text{max}}$.}
\end{center}
%\vspace{-6mm}
\end{table}

\textbf{Real Data.} Next, we study a real dataset consisting of daily sales for 1115 units with 942 days \citep{Rossmannsales2021}. To compare different uncertainty quantification formulas, we consider the \textit{coverage rate maximization} task that aims to maximize the coverage rate given the total interval length constraint. 

\textit{Coverage Rate Maximization.} In particular, given a uncertainty quantification formula, suppose one can provide an interval predictor $[a_{ij}, b_{ij}]$ for each entry $(i, j)$ in a set $\tilde{\Omega}$. The ``coverage'' of $(i, j)$ occurs if $M_{ij} \in [a_{ij}, b_{ij}]$ where $M_{ij}$ is the true value of entry $(i, j).$ The task aims to maximize the coverage rate given that the total length of intervals $\sum_{(i,j)\in \tilde{\Omega}} b_{ij}-a_{ij}$ is constrained by a budget threshold $\alpha$:
\begin{align}
\maximize \quad &\dfrac{  \sum_{(i,j) \in \tilde{\Omega}} \mathds{1}(M_{ij} \in [a_{ij}, b_{ij}]) }{|\tilde{\Omega}|} \nonumber\\
\text{subject to}\quad & \sum_{(i,j)\in \tilde{\Omega}} b_{ij}-a_{ij} \leq \alpha \label{eq:maximize-coverage-task}  
\end{align} 

We are interested in comparing the performances of the above task using different variance predictors $s_{ij}$, either provided by \cref{eq:homogeneous-formula} with the homogeneous Gaussian noise assumption (Theorem 2 in \cite{CFMY:19}), or by our \cref{thm:main-theorem}, capable of addressing the heterogeneous sub-exponential noise. Note that both results in \cite{CFMY:19} and our \cref{thm:main-theorem} predict that $M^{\dd}_{ij} \sim \mathcal{N}(M^{*}_{ij}, s_{ij}^2)$. With this distributional assumption, we tackle \cref{eq:maximize-coverage-task} by a greedy algorithm that achieves the maximal expected coverage rate with the budget constraint. Specifically, with given $\{M^{\dd}_{ij}, s_{ij}\}$, we provide the interval predictors $\{[a_{ij}, b_{ij}]\}$ by solving the following problem:
\begin{align}
\maximize_{a_{ij}, b_{ij}} \quad &\frac{  \sum\limits_{(i,j) \in \tilde{\Omega}} \mathds{E}_{M_{ij} \sim \mathcal{N}(M^{\dd}_{ij}, s_{ij}^2)}(M_{ij} \in [a_{ij}, b_{ij}]) }{|\tilde{\Omega}|} \nonumber\\
\text{subject to}\quad & \sum_{(i,j)\in \tilde{\Omega}} b_{ij}-a_{ij} \leq \alpha \nonumber
\end{align}

\textit{Experiment Results.} In the experiment, the low-rankness of the dataset is verified and the ``true'' rank, as well as the ``true'' underlying matrix $M$, is pre-determined through the spectrum of singular value decomposition. 

We split the entries uniformly into a training set $\Omega$ (with probability $p$) and a test set $\tilde{\Omega}$. We use the observations in $\Omega$ to learn $M^{\dd}$ with \cref{alg:GD}. Let  $M^{\dd} = U^{\dd} \Sigma^{\dd} V^{\dd \top}$ be the SVD of $M^{\dd}$. 

The empirical variance $s_{ij}^{\text{Gaussian}}$ for homogeneous Gaussian noise is computed by \citep{CFMY:19}
\begin{align*}
(s_{ij}^{\text{Gaussian}})^2 = \frac{\hat{\sigma}^2 (\|U^{\dd}_{i,\cdot}\|^2 + \|V^{\dd}_{j,\cdot}\|^2)}{p}
\end{align*}
where $\hat{\sigma}^2 := \sum_{(i,j) \in \Omega} (O_{ij}-M^{\dd}_{ij})^2 / |\Omega|$ is the  empirical estimator for the noise variance. 

We then compute the empirical variance $s_{ij}^{\text{Poisson}}$ for Poisson noise
\begin{align*}
&(s_{ij}^{\text{Poisson}})^2 = \\
&\frac{\sum\limits_{l=1}^{m} M^{\dd}_{lj} \left(\sum\limits_{k=1}^{r} U^{\dd}_{ik}U^{\dd}_{lk}\right)^2 + \sum\limits_{l=1}^{n}M^{\dd}_{il}  \left(\sum\limits_{k=1}^{r} V^{\dd}_{lk} V^{\dd}_{jk}\right)^2}{p}. 
\end{align*}

Given $M^{\dd}, s_{ij}^{\text{Poisson}}$ and $s_{ij}^{\text{Gaussian}}$, the coverage rate maximization task is evaluated in the test set $\tilde{\Omega}.$ The results for various budgets $\alpha$ are reported in \cref{fig:real-data}. The Poisson noise formula shows a higher coverage rate than the homogeneous Gaussian formula, as the former is more robust to addressing heterogeneous noises in sales data. This improvement tends to vanish with more presences of missing entries, which might be due to the degrading accuracy of matrix completion and variance estimation when $p$ decreases.

\begin{figure}[h!]
\center
    \includegraphics[scale=0.5]{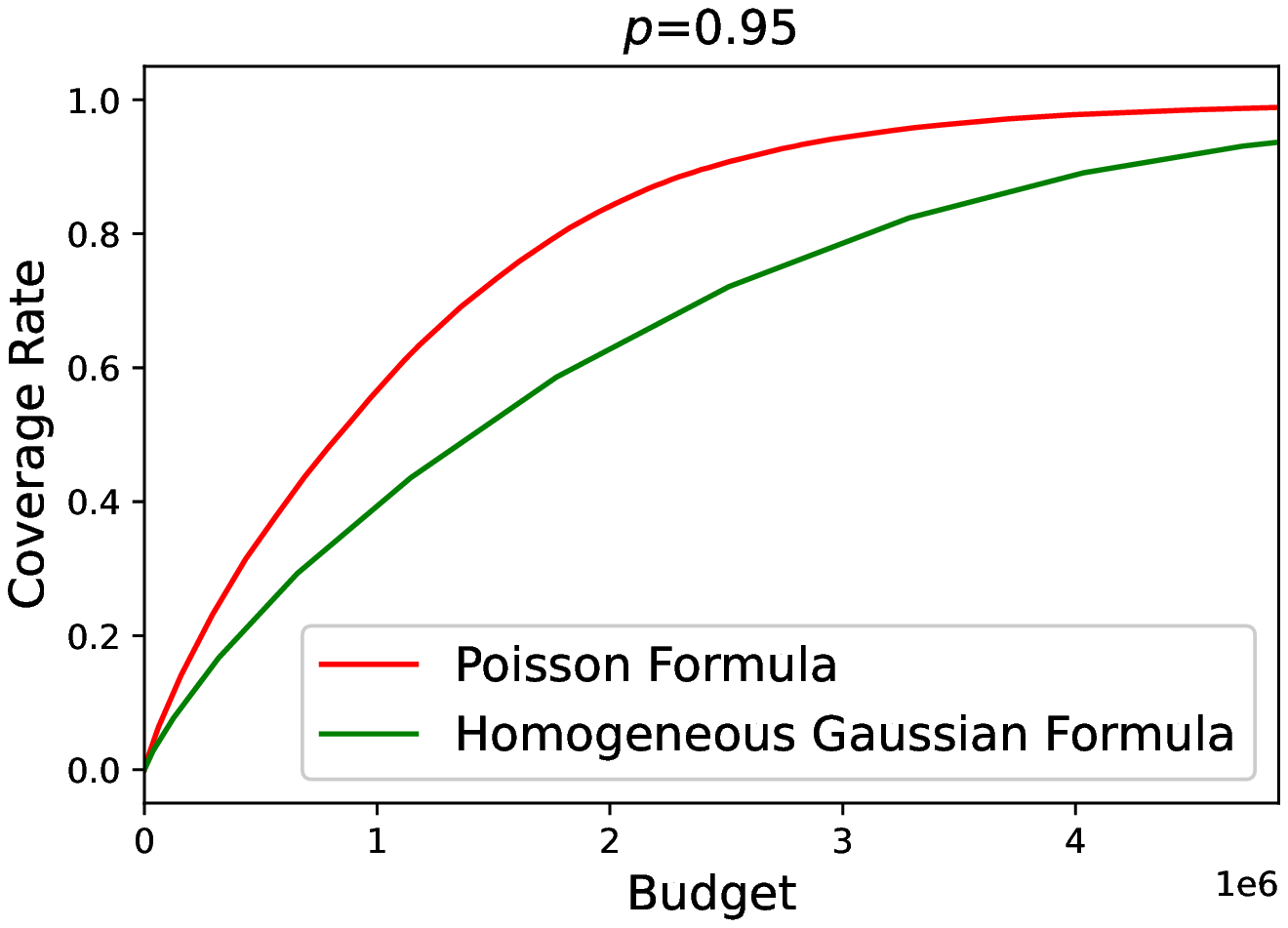}
     \includegraphics[scale=0.5]{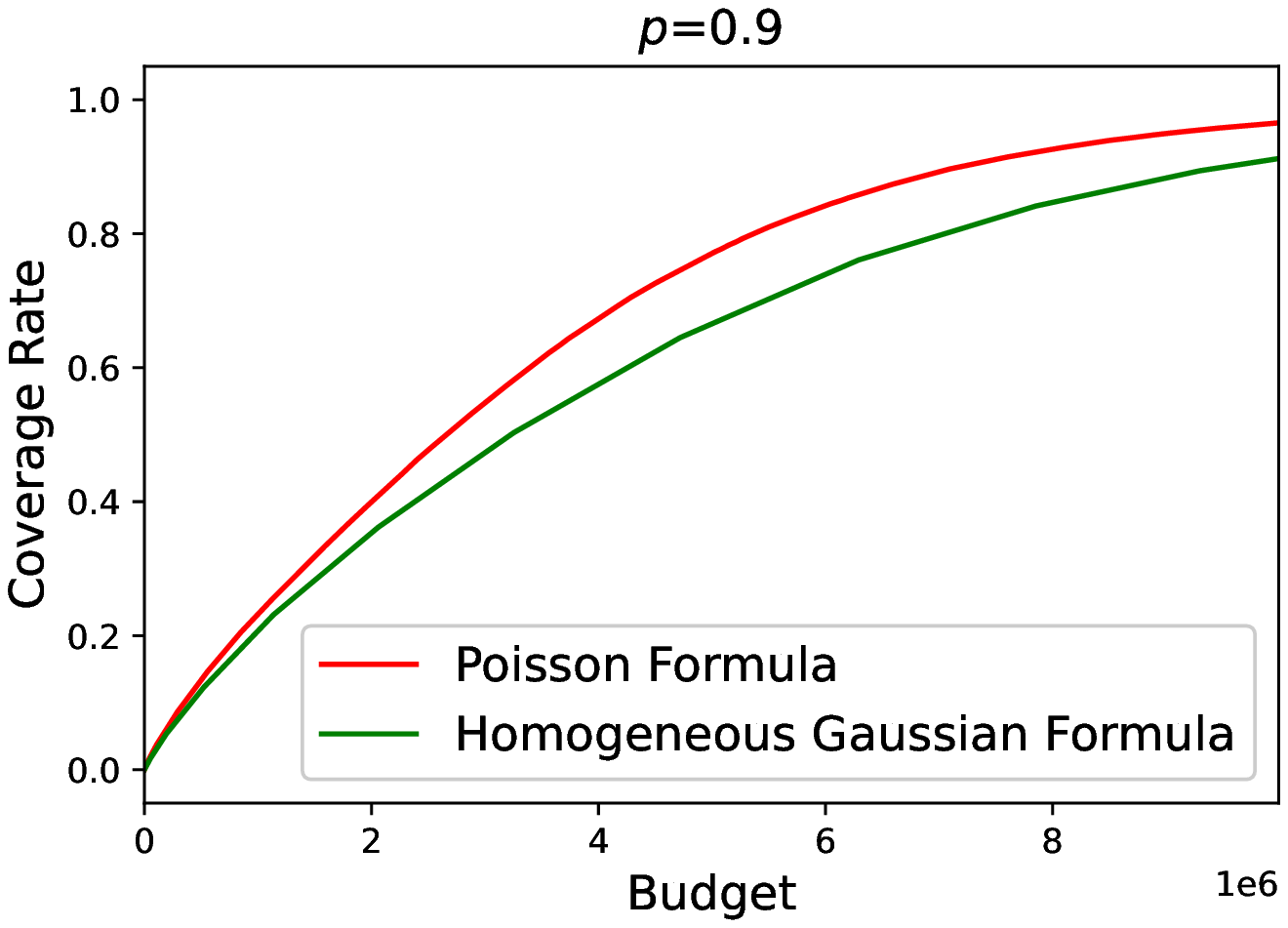}
    \caption{Coverage rates for difference variance formulas correspond to the total-interval-length budget.}
        \label{fig:real-data}
\end{figure}

\section{Conclusion}
We solved the uncertainty quantification problem for matrix completion with heterogeneous and sub-exponential noise. The error variance of a common estimator wsa determined and the asymptotical normality with inference results were established. The explicit formulas for various scenarios such as Poisson noise and Binary noise were analyzed. Experimental results showed significant improvements of our new uncertainty quantification formulas over existing ones.
%
%While our results are somewhat encouraging, t

One exciting direction for further work is in assuming less restrictive $\Omega$. As in most of the matrix completion literature, we made the uniform sampling assumption for $\Omega$, which may not be applicable in some practical applications. The study of uncertainty quantification for matrix completion with non-uniform sampling patterns is especially valuable, given the recent progress on deterministic matrix completion, e.g., \citep{chatterjee2020deterministic}.

\bibliographystyle{informs2014}
\bibliography{reference.bib}

\appendix 
\onecolumn
\begin{appendices}

\section{Proof}\label{sec:append-proof-main-theorem}
\subsection{Leave-one-out sequences}
Similar to \cite{abbe2020entrywise,chen2020noisy,CFMY:19,ma2018implicit}, we employ the leave-one-out techniques. In particular, let $O^{(j)} \in \R^{m\times n}, l = 1, 2, \dotsc, m$ be
\begin{align*}
    O^{(j)}_{ik} := \begin{cases}
    O_{ik} & i\neq j \\
    M^{*}_{ik} & i = j.
    \end{cases}
\end{align*} 
The $O^{(j)}$ is obtained from $O$ by replacing the $j$-th row with corresponding entries of $M^{*}.$ Let $\Omega_{(j)} = \Omega \cup \{(j, k), k=1,2,\dotsc, n\}.$ Consider the observation $P_{\Omega_{(j)}}(O^{(j)})$, i.e., one observes the $j$-th row of $M^{*}$, in additional to $P_{\Omega}(O).$ Consider the non-convex objective function $f^{(j)}$ associated with $P_{\Omega_{(j)}}(O^{(j)})$, denoted by
\begin{align*}
f^{(j)}(X, Y) := \frac{1}{2p} \norm{P_{\Omega_{(j)}}(XY^{\top} - O^{(j)})}_{F}^2 + \frac{\lambda}{2p} \norm{X}_{F}^2 + \frac{\lambda}{2p} \norm{Y}_{F}^2.
\end{align*} 
The gradient descent procedure associated with $f^{(j)}$, similar to \cref{alg:GD}, is denoted below.
\begin{algorithm}[H]
\caption{Gradient Descent with Leave-one-out} \label{alg:GD-leave-one-out}
{\bf Input:} $P_{\Omega_{(j)}}(O^{(j)})$
\begin{algorithmic}[1]

\State{\textbf{{Spectral initialization}}: $X^{0,(j)} = U\sqrt{\Sigma}, Y^{0,(j)}=V\sqrt{\Sigma}$ where  $U\Sigma V$ is the top-$r$ partial SVD decomposition of $\frac{1}{p}P_{\Omega_{(j)}}(O^{(j)})$.} 

\State{\textbf{{Gradient updates}}: \textbf{for }$t=0,1,\ldots,t_{\star}-1$
\textbf{do}
 \begin{subequations}\label{subeq:gradient_update_ncvx-loo}
\begin{align*}
X^{t+1,(j)}= & X^{t,(j)}-\frac{\eta}{p}[P_{\Omega_{(j)}}(X^{t,(j)}Y^{t,(j) \top}-O^{(j)})Y^{t,(j)}+\lambda X^{t,(j)}];\\
Y^{t+1,(j)}= & Y^{t,(j)}-\frac{\eta}{p}[P_{\Omega_{(j)}}(X^{t,(j)}Y^{t,(j) \top}-O^{(j)})^{\top}X^{t,(j)}+\lambda Y^{t,(j)}]
\end{align*}
\end{subequations}
where $\eta$ determines the learning rate. 

\State{\textbf{De-bias}: 
\begin{align*}
    X^{\dd, (j)} &= X^{t_{\star}, (j)} \left(I_r + \frac{\lambda}{p}\left(X^{t_{\star}, (j) \top}X^{t_{\star},(j)}\right)^{-1}\right)^{1/2} \\
    Y^{\dd, (j)} &= Y^{t_{\star}, (j)} \left(I_r + \frac{\lambda}{p}\left(Y^{t_{\star}, (j) \top}Y^{t_{\star},(j)}\right)^{-1}\right)^{1/2}
\end{align*}}
}

\end{algorithmic}
\end{algorithm}
%In particular, the hyperparameters are chosen with $\eta = \Theta(1/(n^{6}\kappa^3\sigma_{\max}))$ and

Similar to the definition associated with rows, for $j=m+1, m+2, \dotsc, m+n$, let $O^{(j)}$ be 
\begin{align*}
    O^{(j)}_{ik} := \begin{cases}
    O_{ik} & k\neq j-m \\
    M^{*}_{ik} & k = j-m.
    \end{cases}
\end{align*} 
The $O^{(j)}$ is obtained from $O$ by replacing the $(j-m)$-th column with corresponding entries of $M^{*}.$ Then $f^{(j)}$, $X^{t,(j)}, Y^{t, (j)}$, for $j=m+1, m+2, \dotsc, m+n$, can be denoted accordingly. 

\subsection{Preliminaries for coarse-grained error guarantees}
\footnote{We assume $n=m$ in the proof. The generalization is straightforward and we omit it for ease of presentation.}Before proceeding, we introduce a set of results that control the estimation error of various variables in a ``coarse-grained'' sense. This enables us to further provide finer bounds in the next sections. The proof of the following results are based on a reduction from centered sub-exponential random variables to centered sub-Gaussian random variables (see \cref{lem:reduction-ExptoGaussian}) and then an invoke of results in \citep{chen2020noisy} (also see Section A.2 in \citep{CFMY:19}). 

For ease of notation, when there is no ambiguity, let $(X, Y) = (X^{t_\star}, Y^{t_\star}), (X^{(j)}, Y^{(j)}) = (X^{t_\star, (j)}, Y^{t_\star, (j)})$ and
\begin{align*}
F := \begin{bmatrix} X\\ Y\end{bmatrix}, F^{\dd} := \begin{bmatrix} X^{\dd}\\ Y^{\dd}\end{bmatrix}, F^{(j)} := \begin{bmatrix} X^{(j)}\\ Y^{(j)}\end{bmatrix}, F^{\dd, (j)} := \begin{bmatrix} X^{\dd, (j)}\\ Y^{\dd, (j)}\end{bmatrix}.
\end{align*}
We also denote the associated rotations accordingly below. 
\begin{align*}
H &:= \arg\min_{R \in \O^{r\times r}} \norm{FR - F^{*}}_{\F}^2\\
H^{(j)} &:= \arg\min_{R \in \O^{r\times r}} \norm{F^{(j)}R - F^{*}}_{\F}^2\\
R^{(j)} &:= \arg\min_{R \in \O^{r\times r}} \norm{F^{(j)}R - FH}_{\F}^2\\
H^{\dd, (j)} &:= \arg\min_{R \in \O^{r\times r}} \norm{F^{\dd, (j)}R - F^{*}}_{\F}^2
\end{align*}
Then, we have the following claims.
\begin{lemma}\label{lem:norm-bounds}Suppose 
\begin{align*}
    n^2 p \gg \kappa^4 \mu^2 r^2 n \log^3 n \text{\quad and \quad } L\log(n)\sqrt{\frac{n}{p}} \ll \frac{\sigma_{\min}}{\sqrt{\kappa^4 \mu r\log n}}.
\end{align*} 
Suppose $\lambda = \Theta(L\log(n)\sqrt{np})$. With probability at least $1-1/O(n^{10})$, we have the following set of results simultaneously. 
\begin{enumerate}
    \item The bounds for $(X, Y)$.
    \begin{subequations}
\begin{align*}
    \norm{FH - F^{*}}_{F} &\lesssim \noise \norm{X^{*}}_{F}  \\
    \norm{FH - F^{*}} &\lesssim \noise \norm{X^{*}} \\
    \norm{FH - F^{*}}_{2,\infty} &\lesssim \kappa \lognoise \norm{F^{*}}_{2,\infty}.
\end{align*}
\end{subequations}
Furthermore, 
\begin{subequations}
\begin{align*}
    \norm{\nabla f(X, Y)}_{F} &\lesssim \frac{1}{n^5} L\log(n) \sqrt{\frac{n}{p}}\sqrt{\sigma_{\min}}\\
    \norm{X^{\top}X - Y^{\top}Y}_{F} &\lesssim \frac{1}{n^5} \noise \sigma_{\max}\\
    \norm{\frac{1}{p}P_{\Omega}(XY^{\top} - X^{*}Y^{*\top}) - (XY^{\top} - X^{*}Y^{*\top})} &\lesssim L \sqrt{\frac{n}{p}} \sqrt{\frac{\kappa^4 \mu^2 r^2\log^3(n)}{np}}. 
\end{align*}
\end{subequations}

\item The bounds for debiased estimator $(X^{\dd}, Y^{\dd})$.
\begin{subequations}
\begin{align}
    \norm{F^{\dd}H - F^{*}} &\lesssim \noise \norm{X^{*}} \nonumber\\
    \norm{F^{\dd}H^{\dd} - F^{*}} &\lesssim \kappa \noise \norm{X^{*}} \label{eq:FdHd-Fstar-opernorm}\\
    \norm{F^{\dd}H^{\dd} - F^{*}}_{F} &\lesssim \noise \norm{X^{*}}_{\F} \nonumber\\
    \norm{F^{\dd}H^{\dd} - F^{*}}_{2,\infty} &\lesssim \kappa \lognoise \norm{F^{*}}_{2,\infty} \label{eq:FdHd-Fstar-rownorm}\\
    \norm{X^{\dd \top}X^{\dd} - Y^{\dd \top}Y^{\dd}} &\lesssim \frac{\kappa}{n^5} \noise \sigma_{\max} \nonumber.
\end{align}
\end{subequations}

\item The bounds for the leave-one-out estimator $(X^{(j)}, Y^{(j)})$.
\begin{subequations}
\begin{align*}
    \norm{F^{(j)}R^{(j)} - FH}_{F} & \lesssim \lognoise \norm{F^{*}}_{2,\infty}\\
    \norm{F^{(j)}H^{(j)} - FH}_{F} & \lesssim \kappa \lognoise \norm{F^{*}}_{2,\infty}\\
    \norm{F^{(j)}H^{(j)} - F^{*}}_{F} & \lesssim \noise \norm{X^{*}}\\
    \norm{F^{(j)}R^{(j)} - F^{*}}_{2,\infty} & \lesssim \kappa \lognoise \norm{F^{*}}_{2,\infty}.
\end{align*}
\end{subequations}

\item The bounds for the leave-one-out version of the debiased estimator $(X^{\dd,(j)}, Y^{\dd,(j)})$. 
\begin{subequations}
\begin{align}
    \norm{F^{\dd,(j)}H^{\dd,(j)} - F^{*}} &\lesssim \kappa \noise \norm{X^{*}} \nonumber\\
    \norm{F^{\dd,(j)}H^{\dd,(j)} - F^{*}}_{2,\infty} &\lesssim \kappa \lognoise \norm{F^{*}}_{2,\infty} \nonumber\\
    \norm{F^{\dd,(j)}H^{\dd,(j)} - F^{\dd}H^{\dd}}_{2,\infty} &\lesssim \kappa \lognoise \norm{F^{*}}_{2,\infty} \label{eq:FdjHdj-FdHd-rownorm}.
\end{align}
\end{subequations}

\item Additional bounds.
\begin{subequations}
\begin{align}
     &\sigma_r(F) \geq 0.5\sqrt{\sigma_{\min}}, \quad \norm{F} \leq 2\norm{X^{*}}, \quad \norm{F}_{\F}  \leq 2\norm{X^{*}}_{\F}, \quad \norm{F}_{2,\infty} \leq 2\norm{F^{*}}_{2,\infty} \nonumber\\
     &\sigma_r(F^{\dd}) \geq 0.5\sqrt{\sigma_{\min}}, \quad \norm{F^{\dd}} \leq 2\norm{X^{*}}, \quad \norm{F^{\dd}}_{\F}  \leq 2\norm{X^{*}}_{\F}, \quad \norm{F^{\dd}}_{2,\infty} \leq 2\norm{F^{*}}_{2,\infty} \label{eq:sigma-r-Fd-bound}\\
     &\sigma_r(F^{(j)}) \geq 0.5\sqrt{\sigma_{\min}}, \quad \norm{F^{(j)}} \leq 2\norm{X^{*}}, \quad \norm{F^{(j)}}_{\F}  \leq 2\norm{X^{*}}_{\F}, \quad \norm{F^{(j)}}_{2,\infty} \leq 2\norm{F^{*}}_{2,\infty}\nonumber \\
     &\sigma_r(F^{\dd,(j)}) \geq 0.5\sqrt{\sigma_{\min}}, \quad \norm{F^{\dd,(j)}} \leq 2\norm{X^{*}}, \quad \norm{F^{\dd,(j)}}_{\F}  \leq 2\norm{X^{*}}_{\F}, \quad \norm{F^{\dd,(j)}}_{2,\infty} \leq 2\norm{F^{*}}_{2,\infty} \label{eq:sigma-r-Fdj-bound}
\end{align}
\end{subequations}
%Furthermore,
%\begin{align}
%    \norm{XY^{\top} - M^{*}}_{\max} &\lesssim \frac{\log n}{\sqrt{n}} \label{eq:max-norm-bound}\\
%    \norm{X^{j}Y^{jT} - M^{*}}_{\max} &\lesssim \frac{\log n}{\sqrt{n}} \label{eq:max-norm-bound-j}.
%\end{align}
\end{enumerate}
\end{lemma}
\begin{proof}
The above bounds for sub-Gaussian noise have been shown in \cite{CFMY:19}. To generalize these bounds to sub-exponential noise, we observe that for any sub-exponential zero-mean random variable $X$ with $\|X\|_{\phi_1} \leq L$, one can construct a sub-Gaussian zero-mean random variable $Y$ with $\|Y\|_{\phi_2} \lesssim L\log(n)$, and $Y$ is extremely close to $X$ where $\Pr(X\neq Y) = 1/\text{ploy}(n)$ (see \cref{lem:reduction-ExptoGaussian}). 

Then, for a noise matrix $E \in R^{m\times n}$ with independent sub-exponential entries $\|E_{ij}\|_{\psi_1}\leq L$, we can construct $E' \in \R^{m\times n}$ with independent sub-Gaussian entries $\|E'_{ij}\|_{\psi_2} \lesssim L\log(n)$ and $\Pr(E\neq E') = 1/\text{poly}(n).$ Then one can employ the results in Section A.2 of \cite{CFMY:19} for $E'$ (with an $O(\log(n))$ increase of sub-Gaussian norm) to provide bounds for $E$, which completes the proof. 
\end{proof}

\subsection{Characterization of low-rank factors and proof of  \cref{thm:eigen-vector}}\label{sec:eigen-vector}
%In this section, we state and prove the following theorem for characterizing the error of singular-vector estimations. See \cref{sec:proof-main-theorem} for the proof of \cref{thm:main-theorem} based on  \cref{thm:eigen-vector} below.
%\begin{theorem}\label{thm:eigen-vector}
%Assume $n^2 p \gg \kappa^4 \mu^2 r^2 n \log^3 n$ and $L\log(n)\sqrt{\frac{n}{p}} \ll \frac{\sigma_{\min}}{\sqrt{\kappa^4 \mu r\log n}}$. There exists a rotation matrix $H^{\dd} \in \O^{r\times r}$ and $\Phi_X \in \R^{n\times r}, \Phi_Y \in \R^{n\times r}$ such that the following holds with probability $1-O(n^{-10})$,
%\begin{align*}
%    X^{\dd}H^{\dd} - X^{*} &= \frac{1}{p}P_{\Omega}(E)Y^{*}(Y^{*T}Y^{*})^{-1} + \Phi_{X}\\
%    Y^{\dd}H^{\dd} - Y^{*} &= \frac{1}{p}P_{\Omega}(E)^{\top}X^{*}(X^{*T}X^{*})^{-1} + \Phi_{Y}
%\end{align*}
%where 
%\begin{align*}
%\max\left\{\norm{\Phi_X}_{2,\infty}, \norm{\Phi_Y}_{2,\infty}\right\} \lesssim \frac{\sigmaL}{\sqrt{p\sigma_{\min}}}\left( \kappa \frac{\sigmaL}{\sigma_{\min}} \sqrt{\frac{\kappa^7 \mu r n\log(n)}{p}} + \sqrt{\frac{\kappa^{7}\mu^3 r^{3}\log^2 n}{np}}\right).
%\end{align*}
%\end{theorem}

We will prove \cref{thm:eigen-vector} based on a decomposition of $X^{\dd} H^{\dd} - X^{*}$ and $Y^{\dd}H^{\dd} - Y^{*}$. Due to the symmetry, we focus on $X^{\dd}H^{\dd} - X^{*}$. From \cite{CFMY:19}, we have the following characterization by direct algebra.
\begin{lemma}[Eq. (5.11) \cite{CFMY:19}] 
Let $\bar{X}^{\dd} = X^{\dd}H^{\dd}, \bYd = Y^{\dd}H^{\dd}.$ Then 
\begin{align*}
X^{\dd}H^{\dd} - X^{*} = \frac{1}{p}P_{\Omega}(E)Y^{*}(Y^{*T}Y^{*})^{-1} + \Phi_{X}
\end{align*}
where 
\begin{align*}
\Phi_{X} 
&:= 
\underbrace{\frac{1}{p}P_{\Omega}(E)\left[\bYd(\bYdT\bYd)^{-1}-Y^{*}(Y^{*T}Y^{*})^{-1}\right]}_{\Phi_1} + \underbrace{X^{*}\left[Y^{*T}\bYd(\bYdT\bYd)^{-1}-I_{r}\right]}_{\Phi_2} + \underbrace{A\bYd(\bYdT\bYd)^{-1}}_{\Phi_3}\nonumber\\
&\quad  + \underbrace{\nabla_{X} f(X, Y) (Y^{\top}Y)^{-1} \left(I_r + \frac{\lambda}{p} (Y^{\top}Y)^{-1}\right)^{1/2}(Y^{\dd T}Y^{\dd})^{-1} H^{\dd} + X\Dbalan H^{\dd}}_{\Phi_4}
\end{align*}
with
\begin{align}
    A &:= \frac{1}{p} P_{\Omega} (XY^{\top} - X^{*}Y^{*T}) - (XY^{\top} - X^{*}Y^{*T})\\
    \Dbalan &:= \left(I_r + \frac{\lambda}{p} (X^{\top}X)^{-1}\right)^{1/2} - \left(I_r + \frac{\lambda}{p}(Y^{\top}Y)^{-1}\right)^{1/2}.
\end{align}
\end{lemma}

% \begin{proof}
% Observe that
% \begin{align}
%     Y(Y^{\top}Y)^{-1} H = YH(H^{\top}Y^{\top}YH)^{-1} = \hat{Y}(\hat{Y}^{\top}\hat{Y})^{-1}.
% \end{align}
% Therefore, right-multiplying by $H$ on the \cref{eq:X=1p} and substracting $X^{*}$, we can obtain the desired result.
% \end{proof}

Then, to show \cref{thm:eigen-vector}, it boils down to bound $\|\Phi_{X}\|_{2,\infty}$. Since $\Phi_{X}=\Phi_1+\Phi_2+\Phi_3+\Phi_4$, it is sufficient to prove the following lemma. 
\begin{lemma}\label{lem:bound-of-Phi}
With probability $1-O(n^{-10})$, we have
\begin{align}
    \norm{\Phi_1}_{2,\infty} &\lesssim \frac{\sigmaL}{\sqrt{p\sigma_{\min}}}\frac{\sigmaL}{\sigma_{\min}} \sqrt{\frac{\kappa^3 \mu r n\log n}{p}} \label{eq:Phi1-rownorm}\\
    \norm{\Phi_2}_{2,\infty} &\lesssim \frac{\sigmaL}{\sqrt{p\sigma_{\min}}}\left(\kappa \frac{\sigmaL}{\sigma_{\min}} \sqrt{\frac{\kappa^7 \mu r n}{p}} + \sqrt{\frac{\kappa^{7}\mu^3 r^{3}\log n}{np}}\right) \label{eq:Phi2-rownorm}\\
    \norm{\Phi_3}_{2, \infty} &\lesssim \frac{\sigmaL}{\sqrt{p\sigma_{\min}}} \sqrt{\frac{\kappa^5 \mu^3 r^3 \log^2 n}{np}} \label{eq:Phi3-rownorm}\\
    \norm{\Phi_4}_{2,\infty} &\lesssim \frac{\sigmaL}{\sqrt{p\sigma_{\min}}} \frac{1}{n^4} \label{eq:Phi4-rownorm}.
\end{align}
\end{lemma}

\subsection{Proof of \cref{lem:bound-of-Phi}}
Equipped with \cref{lem:norm-bounds}, the proof of \cref{lem:bound-of-Phi} follows a similar framework for the proof of Lemma 5 to Lemma 8 in \citep{CFMY:19}, where the differences are highlighted below. We extend the homogeneous Gaussian noise in \citep{CFMY:19} to the heteroskedastic and subexponential noise, based on the results from sub-exponential matrix completion \citep{mcrae2019low} and a subexponential variant of Matrix Bernstein inequality (see \cref{lem:subexponential-Bernstein}). 

\subsubsection{Proof of \cref{eq:Phi1-rownorm}}
By triangle inequality, for any row of $\Phi_1$, we have 
\begin{align*}
\norm{e_j^{\top} \Phi_1}_2 &\leq \underbrace{\norm{e_j^{\top}\frac{1}{p} P_{\Omega}(E)\left[\bar{Y}^{\dd,(j)}(\bar{Y}^{\dd,(j) T}\bar{Y}^{\dd,(j)})^{-1}-Y^{*}(Y^{*T}Y^{*})^{-1}\right]}_2}_{\alpha_1}  \\
&\quad + \underbrace{\norm{e_j^{\top} \frac{1}{p} P_{\Omega}(E)\left[\bYd(\bYdT\bYd)^{-1} - \bar{Y}^{\dd,(j)}(\bar{Y}^{\dd,(j) T}\bar{Y}^{\dd,(j)})^{-1}\right]}_2}_{\alpha_2}
\end{align*}
where $\bar{Y}^{\dd,(j)} = Y^{\dd, (j)} H^{\dd, (j)}, j=1,2,\dotsc, n.$

\begin{enumerate}
    \item Next, we control $\alpha_1$. Let $\Delta^{(j)} = \bar{Y}^{\dd,(j)}(\bar{Y}^{\dd,(j) T}\bar{Y}^{\dd,(j)})^{-1}-Y^{*}(Y^{*T}Y^{*})^{-1}$ ($\Delta^{(j)} \in R^{n\times r}$). Then
    \begin{align*}
        \alpha_1 = \norm{e_j^{\top}\frac{1}{p}P_{\Omega}(E) \Delta^{(j)}}_2.  
    \end{align*}
    
    Furthermore, we have the following claim. 
    \begin{claim} With probability $1-O(n^{-11})$, 
    \begin{align*}
    \norm{\Delta^{(j)}} &\lesssim \frac{1}{\sqrt{\sigma_{\min}}} \cdot \frac{\sigmaL}{\sigma_{\min}} \sqrt{\frac{\kappa^3 n}{p}}\\
    \norm{\Delta^{(j)}}_{2,\infty}  &\lesssim \frac{1}{\sqrt{\sigma_{\min}}} \cdot \frac{\sigmaL}{\sigma_{\min}} \sqrt{\frac{\kappa^5 \mu r \log(n)}{p}}.
    \end{align*}
    \end{claim}
    \begin{proof}
    The result can be derived from \cref{lem:norm-bounds} (see the similar derivation for Claim 2 in \cite{CFMY:19}). 
    \end{proof}

In order to control $\alpha_1$, let $z = e_j^{\top}\frac{1}{p}P_{\Omega}(E) \Delta^{(j)}$. Hence $\norm{z}_2 = \alpha_1$. Note that in \cite{CFMY:19}, $z$ is a \textit{Gaussian random vector} due to the Gaussianity of $E$ (conditioned on $\Omega$ and $\Delta^{(j)}$), therefore $\norm{z}_2$ can be easily bounded. Instead, here we need to fine-tune the bound for $z$ using the \textit{subexponential} property of $E$. In particular, for $l \in [r]$, we have 
\begin{align*}
    z_l = \frac{1}{p}\sum_{k=1}^{n} \delta_{j,k} E_{j, k} \Delta_{k, l}^{(j)}
\end{align*}
where $\delta_{j, k} \in \{0, 1\}$ indicates whether $(j, k) \in \Omega.$ Note that $E_{j,k}$ and $\delta_{j,k}$ is independent from $\Delta^{(j)}$ by the construction of $Y^{\dd, (j)}$. Next, construct $A_{k} \in \R^{r}, k=1, 2, \dotsc, n$ where $A_{k,l} = E_{j,k} \delta_{j,k} \Delta_{k,l}$ (note that $z = \frac{1}{p}\sum_{k=1}^{n} A_{k}$). 

Then, condition on $\Delta^{(j)}$, note that $\delta_{j,k}$ and $E_{j,k}$, $k=1,2,\dotsc, n$, are independent Bernoulli and sub-exponential random variables respectively. We are able to apply a variant of Matrix Bernstein inequality (\cref{lem:subexponential-Bernstein}) to $A_{k}, k=1,2,\dotsc,n.$ In particular, note that 
\begin{align*}
V &:= \max\left(\norm{\sum_{k=1}^{n} E[A_{k}A_{k}^{\top}]}, \norm{\sum_{k=1}^{n} E[A_{k}^{\top}A_{k}]}\right) \leq p L^2 \norm{\Delta^{(j)}}_{\F}^2 \leq p L^2 r \norm{\Delta^{(j)}}^2\\
B &:= \max_{1\leq k\leq n}\norm{\norm{A_{k}}}_{\psi_1} \leq L \norm{\Delta^{(j)}}_{2,\infty}
\end{align*}
Apply \cref{lem:subexponential-Bernstein} to $A_k$, with probability $1-O(n^{-11}),$ we have
\begin{align*}
    \norm{\sum_{k=1}^{n} A_{k}}
    &\lesssim \sqrt{V\log(n)} + B\log^2(n)\\
    &\lesssim L \sqrt{rp} \norm{\Delta^{(j)}} \sqrt{\log(n)} + L \norm{\Delta^{(j)}}_{2,\infty} \log^2(n) \\ 
    &\lesssim L \sqrt{rp} \frac{1}{\sqrt{\sigma_{\min}}} \cdot \frac{\sigmaL}{\sigma_{\min}} \sqrt{\frac{\kappa^3 n\log(n)}{p}} + \frac{L}{\sqrt{\sigma_{\min}}} \cdot \frac{\sigmaL}{\sigma_{\min}} \sqrt{\frac{\kappa^5 \mu r \log(n)}{p}} \log^2(n) \\
    &\lesssim \frac{L}{\sqrt{\sigma_{\min}}} \cdot \frac{\sigmaL}{\sigma_{\min}} \sqrt{\kappa^3 n r \log^2(n)}
\end{align*}
where we have utilized $n^2 p \gg \kappa^4 \mu^2 r^2 n \log^3 n$ in the last inequality. This provides a desired bound for $\alpha_1$. 
\begin{align}
    \alpha_1 = \norm{z}_2 = \frac{1}{p}  \norm{\sum_{k=1}^{n} A_{k}} \lesssim \frac{L}{p} \frac{1}{\sqrt{\sigma_{\min}}} \cdot \frac{\sigmaL}{\sigma_{\min}} \sqrt{\kappa^3 n r \log^2(n)}
\end{align}

\item Next, we control $\alpha_2.$ Note that
\begin{align*}
    \alpha_2 
    &= \norm{e_j^{\top} \frac{1}{p} P_{\Omega}(E)\left[\bYd(\bYdT\bYd)^{-1} - \bar{Y}^{\dd,(j)}(\bar{Y}^{\dd,(j) T}\bar{Y}^{\dd,(j)})^{-1}\right]}_2\\
    &\leq \frac{1}{p}\norm{P_{\Omega}(E)} \norm{\bYd(\bYdT\bYd)^{-1} - \bar{Y}^{\dd,(j)}(\bar{Y}^{\dd,(j) T}\bar{Y}^{\dd,(j)})^{-1}}
\end{align*}
Different from \cite{CFMY:19}, we need to bound $\norm{P_{\Omega}(E)}$ where $E$ is a sub-exponential instead of Gaussian matrix. We obtain the bound by invoking a recent result on sub-exponential matrices from \cite{mcrae2019low}. In particular, by \cref{lem:operator-norm-bound}, we have, with probability $1-O(n^{-11})$,
\begin{align*}
    \frac{1}{p}\norm{P_{\Omega}(E)} \lesssim L\sqrt{\frac{n}{p}}.
\end{align*}

Then, we have
\begin{align*}
    \alpha_2 
    &\leq L\sqrt{\frac{n}{p}}\norm{\bYd(\bYdT\bYd)^{-1} - \bar{Y}^{\dd,(j)}(\bar{Y}^{\dd,(j) T}\bar{Y}^{\dd,(j)})^{-1}} \\
    &\overset{(i)}{\lesssim} L\sqrt{\frac{n}{p}}\max\left(\norm{\bYd(\bYdT\bYd)^{-1}}^2, \norm{\bar{Y}^{\dd,(j)}(\bar{Y}^{\dd,(j) T}\bar{Y}^{\dd,(j)})^{-1}}^2\right)\norm{\bYd-\bar{Y}^{\dd, (j)}}\\
    &\overset{(ii)}{\lesssim} L\sqrt{\frac{n}{p}}\frac{1}{\sigma_{\min}} \norm{\bYd-\bar{Y}^{\dd, (j)}}\\
    &\overset{(iii)}{\lesssim} L\sqrt{\frac{n}{p}}\frac{1}{\sigma_{\min}} \kappa \lognoise \sqrt{\frac{\mu r \sigma_{\max}}{n}}\\
    &\lesssim L\sqrt{\frac{n}{p}}\frac{1}{\sqrt{\sigma_{\min}}} \frac{L\log(n)}{\sigma_{\min}} \sqrt{\frac{\kappa^3 \mu r \log(n)}{p}}.
\end{align*}
where, similarly as \cite{CFMY:19}, (i) is due to \cref{lem:pseudo-inverse} (the perturbation bound for pseudo-inverse), (ii) is due to \cref{eq:sigma-r-Fdj-bound,eq:sigma-r-Fd-bound}, and (iii) is due to \cref{eq:FdjHdj-FdHd-rownorm} and $\norm{F^{*}}_{2,\infty} \leq \sqrt{\frac{\mu r \sigma_{\max}}{n}}.$

\end{enumerate}

Then, combining the control of $\alpha_1$ and $\alpha_2$, we have
\begin{align*}
\norm{e_j^{\top} \Phi_1}_2 \leq \alpha_1 + \alpha_2
&\lesssim \frac{1}{p}\frac{L\log(n)}{\sqrt{\sigma_{\min}}} \cdot \frac{\sigmaL}{\sigma_{\min}} \sqrt{\kappa^3 nr} + L\sqrt{\frac{n}{p}}\frac{1}{\sqrt{\sigma_{\min}}} \frac{L\log(n)}{\sigma_{\min}} \sqrt{\frac{\kappa^3 \mu r \log(n)}{p}} \\
&\lesssim \frac{L\log(n)}{\sqrt{p\sigma_{\min}}} \frac{L\log(n)}{\sigma_{\min}} \sqrt{\frac{\kappa^3 \mu r n\log(n)}{p}}.
\end{align*}
This establishes the bound \cref{eq:Phi1-rownorm} by taking the maximum over $j \in [n]$ and the union bound.

\subsubsection{Proof of \cref{eq:Phi2-rownorm}}
Note that $I_{r} = \bYdT\bYd(\bYdT\bYd)^{-1}.$ We have
\begin{align}
    \norm{e_j^{\top}\Phi_2}
    &= \norm{e_j^{\top} X^{*} (Y^{*T}-\bYdT)\bYd(\bYdT\bYd)^{-1}} \\
    &\leq \norm{X^{*}}_{2,\infty} \norm{ (Y^{*} - \bYd)^{\top}\bYd} \norm{(\bYdT\bYd)^{-1}} \\
    &\overset{(i)}{\lesssim} \sqrt{\frac{\kappa \mu r}{n}} \frac{1}{\sqrt{\sigma_{\min}}} \norm{ (Y^{*} - \bYd)^{\top}\bYd} \label{eq:ejT-Phi2}
\end{align}
where (i) is due to $\norm{X^{*}}_{2,\infty} \leq \sqrt{\frac{\mu r \sigma_{\max}}{n}}$ and \cref{eq:sigma-r-Fd-bound}. To control $\norm{ (Y^{*} - \bYd)^{\top}\bYd}$, we have the following claim.
\begin{claim}
\begin{align}
    \norm{ (Y^{*} - \bYd)^{\top}\bYd} \lesssim \frac{1}{\sigma_{\min}} \underbrace{\norm{\bXdT \frac{1}{p}P_{\Omega}(E)Y^{*}}}_{\alpha_1} + \frac{1}{\sigma_{\min}} \alpha_2 + \kappa \alpha_3 + \frac{1}{\sigma_{\min}} \alpha_4 \label{eq:Ystar-bYd}
\end{align}
where 
\begin{align*}
    \alpha_2 &\lesssim \sigma_{\max} \sigmaL \sqrt{\frac{n}{p}} \sqrt{\frac{\kappa^{4}\mu^2r^2\log(n)}{np}} \\
    \alpha_3 &\lesssim \left(\kappa \noise \right)^2 \sigma_{\max}\\
    \alpha_4 &\lesssim \frac{\kappa}{n^5} \noise \sigma_{\max}^2.
\end{align*}
\end{claim}
\begin{proof}[Proof]
This can be derived based on \cref{lem:norm-bounds}, following a similar derivation as Section D.3 in \citep{CFMY:19}, 
\end{proof}

What remains to control is $\norm{\bXdT \frac{1}{p}P_{\Omega}(E)Y^{*}}$, where the subexponential property of $E$ needs to be addressed. Note that
\begin{align*}
    \norm{\bXdT \frac{1}{p}P_{\Omega}(E)Y^{*}} 
    &\leq \norm{X^{* T} \frac{1}{p}P_{\Omega}(E) Y^{*}} + \norm{\bXd-X^{*}} \norm{ \frac{1}{p}P_{\Omega}(E)}\norm{Y^{*}}\\
    &\overset{(i)}{\lesssim} \norm{X^{*} \frac{1}{p}P_{\Omega}(E) Y^{*}} + \kappa \noise \sigma_{\max} \norm{ \frac{1}{p}P_{\Omega}(E)}\\
    &\overset{(ii)}{\lesssim} \norm{X^{*} \frac{1}{p}P_{\Omega}(E) Y^{*}} + L^2\log(n)\frac{n}{p}\kappa^2. \label{eq:bxdT-bound}
\end{align*}
where (i) is due to \cref{eq:FdHd-Fstar-opernorm} and $\norm{X^{*}} = \sqrt{\sigma_{\max}}, \norm{Y^{*}} = \sqrt{\sigma_{\max}}$, and (ii) is due to \cref{lem:operator-norm-bound}.

In order to bound $\norm{X^{*T} \frac{1}{p}P_{\Omega}(E) Y^{*}}$, we will invoke the subexponential version of Matrix Bernstein inequality (\cref{lem:subexponential-Bernstein}). To begin, note that 
\begin{align*}
    X^{*T} \frac{1}{p}P_{\Omega}(E) Y^{*} = \frac{1}{p}\sum_{k=1}^{n} \sum_{l=1}^{n} X^{*}_{k,\cdot} Y_{l,\cdot}^{*\top} \delta_{k,l} E_{k, l}
\end{align*}
where $\delta_{k,l} \sim \Ber(p)$ indicates whether $(k, l) \in \Omega.$ Then, let $A_{k,l} := X^{*}_{k,\cdot} Y^{*\top}_{l,\cdot} \delta_{k,l} E_{k, l}$ for $k\in[n], l\in[n]$ ($A_{k,l} \in \R^{r\times r}$). Note that
\begin{align*}
V &:= \max\left(\norm{\sum_{k,l} E[A_{k,l}A_{k,l}^{\top}]}, \norm{\sum_{k,l} E[A_{k,l}^{\top}A_{k,l}]}\right) \leq p L^2 \norm{X^{*}}_{\F}^2\norm{Y^{*}}_{\F}^2\\
B &:= \max_{k,l}\norm{\norm{A_{k,l}}}_{\psi_1} \leq L \norm{X^{*}}_{2,\infty} \norm{Y^{*}}_{2,\infty}
\end{align*}
Then apply \cref{lem:subexponential-Bernstein}, with probability $1-O(n^{-11})$, we have
\begin{align*}
\norm{\sum_{k,l} A_{k,l}} 
&\lesssim \sqrt{V \log(n)} + B\log^{2}(n)\\
&\leq \sqrt{p}L \sigma_{\max} r \sqrt{\log(n)} + \frac{\mu r}{n} L \sigma_{\max} \log^{2}(n)\\
&\overset{(i)}{\lesssim} \sqrt{p}L \sigma_{\max} r \sqrt{\log(n)}
\end{align*}
where in (i) we use that $n^2 p \gg \kappa^4 \mu^2 r^2 n \log^3 n.$ This then implies 
\begin{align}
     \norm{X^{*} \frac{1}{p}P_{\Omega}(E) Y^{*}} 
     &\leq  \frac{1}{p}\norm{\sum_{k,l} A_{k,l}}  \\
     &\lesssim \frac{r}{\sqrt{p}}  L \sigma_{\max} \sqrt{\log(n)} \label{eq:Xstar-Ystar-E}.
\end{align}
Plug \cref{eq:Xstar-Ystar-E} into \cref{eq:bxdT-bound} and combine with \cref{eq:Ystar-bYd}, we arrive at
\begin{align*}
    \norm{ (Y^{*} - \bYd)^{\top}\bYd} 
    &\lesssim \frac{1}{\sigma_{\min}} \alpha_1 + \frac{1}{\sigma_{\min}} \alpha_2 + \kappa \alpha_3 + \frac{1}{\sigma_{\min}} \alpha_4\\
    &\lesssim \frac{1}{\sigma_{\min}} \left(\frac{r}{\sqrt{p}}  L \sigma_{\max} \sqrt{\log(n)} +  L^2\log(n)\frac{n}{p}\kappa^2 \right) \\
    &\quad + \frac{1}{\sigma_{\min}} \sigma_{\max} \sigmaL \sqrt{\frac{n}{p}} \sqrt{\frac{\kappa^{4}\mu^2r^2\log(n)}{np}} \\
    &\quad + \kappa \left(\kappa \noise \right)^2 \sigma_{\max} + \frac{1}{\sigma_{\min}}  \frac{\kappa}{n^5} \noise \sigma_{\max}^2\\
    &\lesssim \kappa \left(\kappa \noise \right)^2 \sigma_{\max} + \kappa \frac{L}{p} \sqrt{\kappa^4 \mu^2 r^2 \log^3(n)}.
\end{align*}
Plug this back into \cref{eq:ejT-Phi2} and take the maximum over $1\leq j \leq n$, we finish the proof.
\begin{align*}
   \norm{\Phi_2}_{2,\infty} 
   &\lesssim \sqrt{\frac{\kappa \mu r}{n}} \frac{1}{\sqrt{\sigma_{\min}}} \left( \kappa \left(\kappa \noise \right)^2 \sigma_{\max} + \kappa \frac{L}{p} \sqrt{\kappa^4 \mu^2 r^2 \log^3(n)}\right)\\ 
   &\lesssim \frac{L\log(n)}{\sqrt{p\sigma_{\min}}} \left(\kappa \frac{L\log(n)}{\sigma_{\min}} \sqrt{\frac{\kappa^7 \mu r n}{p}} + \sqrt{\frac{\kappa^7\mu^3r^3\log(n)}{np}}\right).
\end{align*}

The proof of \cref{eq:Phi3-rownorm} and \cref{eq:Phi4-rownorm} follow the similar derivations of section D.4 and D.5 in \cite{CFMY:19}. We omit the proof for brevity. 

\subsection{Proof of \cref{thm:main-theorem}}\label{sec:proof-main-theorem}
Next, we provide a proof of \cref{thm:main-theorem} based on \cref{thm:eigen-vector}. Note that
\begin{align*}
    M^{\dd}_{ij} - M^{*}_{ij}
    &= (X^{\dd}H^{\dd}H^{\dd T}Y^{\dd T})_{ij} - (X^{*}Y^{*T})_{ij} \\
    &= e_{i}^{\top}X^{*}(Y^{\dd}H^{\dd}-Y^{*})^{\top}e_{j} + e_i^{\top}(X^{\dd}H^{\dd} - X^{*})Y^{*T}e_{j} + e_{i}^{\top}(X^{\dd}H^{\dd} - X^{*}) (Y^{\dd}H^{\dd}-Y^{*})^{\top}e_{j}\\
    &\overset{(i)}{=} e_i^{\top} \frac{1}{p}P_{\Omega}(E)V^{*}V^{*T}e_j + e_i^{\top}\frac{1}{p}U^{*}U^{*T}P_{\Omega}(E)e_j + e_i^{\top}\Phi_{X} Y^{*T}e_{j} + e_i^{\top}X^{*}\Phi_{Y}^{\top}e_j + e_i^{\top}\Delta_{X} \Delta_{Y}^{\top}e_j\\
    &\overset{(ii)}{=} \underbrace{\frac{1}{p}\sum_{l=1}^{n} \delta_{il}E_{il}  \left(\sum_{k=1}^{r} V^{*}_{lk} V^{*}_{jk}\right) + \frac{1}{p}\sum_{l=1, l\neq i}^{n} \delta_{lj} E_{lj} \left(\sum_{k=1}^{r} U^{*}_{ik}U^{*}_{lk}\right) + \delta_{ij}'E'_{ij} \left(\sum_{k=1}^{r} U^{*}_{ik}U^{*}_{lk}\right)}_{\epsilon^{(1)}_{ij}}\\
    &\quad + \underbrace{(\delta_{ij}E_{ij} -\delta_{ij}'E'_{ij})\left(\sum_{k=1}^{r} U^{*}_{ik}U^{*}_{lk}\right) + e_i^{\top}\Phi_{X} Y^{*T}e_{j} + e_i^{\top}X^{*}\Phi_{Y}^{\top}e_j + e_i^{\top}\Delta_{X} \Delta_{Y}^{\top}e_j}_{\epsilon^{(2)}_{ij}}
\end{align*}
where in (i) we use the notation in \cref{thm:eigen-vector} and $\Delta_{X} := X^{\dd}H^{\dd} - X^{*}$ and $\Delta_{Y} := Y^{\dd}H^{\dd}-Y^{*}$, in (ii) we introduce the exogenous variables $\delta_{ij}' \sim \Ber(p), E'_{ij} \overset{d}{=}E_{ij}$.

By \cref{thm:eigen-vector}, we have
\begin{align*}
|e_i^{\top}\Phi_{X} Y^{*T}e_{j}|  
&\lesssim \norm{\Phi_{X}}_{2,\infty} \norm{Y^{*}_{j,\cdot}} \\
&\lesssim \norm{V^{*}_{j,\cdot}}\sqrt{\kappa}\frac{\sigmaL}{\sqrt{p}}\left( \kappa \frac{\sigmaL}{\sigma_{\min}} \sqrt{\frac{\kappa^7 \mu r n\log(n)}{p}} + \sqrt{\frac{\kappa^{7}\mu^3 r^{3}\log^2 n}{np}}\right).
\end{align*}
By symmetry, we can obtain the similar bound for $e_i^{\top}X^{*}\Phi_{Y}^{\top}e_j.$ Also note that, by \cref{eq:FdHd-Fstar-rownorm}, we have
\begin{align*}
|e_i^{\top} \Delta_{X} \Delta_{Y}^{\top} e_j| 
&\leq \norm{\Delta_{X}}_{2,\infty} \norm{\Delta_{Y}}_{2,\infty}\\
&\lesssim  \left(\kappa \lognoise \norm{F^{*}}_{2,\infty}\right)^2\\
&= \left( \frac{L\log(n)}{\sqrt{\sigma_{\min}}} \sqrt{\frac{\kappa^3 \mu r \log(n)}{p}}\right)^2.
\end{align*}
This then implies, with probability $1-O(n^{-10})$, 
\begin{align*}
|\epsilon_{ij}^{(2)}| 
&\lesssim  (\norm{V^{*}_{j,\cdot}} + \norm{U^{*}_{i,\cdot}})\frac{\sigmaL}{\sqrt{p}}\left( \frac{\sigmaL}{\sigma_{\min}} \sqrt{\frac{\kappa^{10} \mu r n\log(n)}{p}} + \sqrt{\frac{\kappa^{8}\mu^3 r^{3}\log^2 n}{np}}\right) \\
&\quad + \left( \frac{L\log(n)}{\sqrt{\sigma_{\min}}} \sqrt{\frac{\kappa^3 \mu r \log(n)}{p}}\right)^2 + L\log(n) \frac{\mu r}{n}\\
&\lesssim \underbrace{\frac{L^2\log^{3}(n)\mu r \kappa^5}{p\sigma_{\min}} + \frac{L\mu^2 r^2 \log^{2}(n)\kappa^4}{pn}}_{D_{ij}}
\end{align*}

Next, we analyze $\epsilon_{ij}^{(1)}$. Note that one can view $\epsilon^{(1)}_{ij}=\sum_{l=1}^{2n} x_l$ as the sum of independent zero-mean random variables $x_l$. Hence, one can apply the Berry-Esseen type of bounds. Consider $s_{ij}^2 = \sum_{l} E[x_l^2], \rho = \sum_{l} E[x_l^3]$. Then
\begin{align*}
s_{ij}^2 
&:= \frac{1}{p} \sum_{l=1}^{n}\sigma_{il}^2  \left(\sum_{k=1}^{r} V^{*}_{lk} V^{*}_{jk}\right)^2 + \frac{1}{p}\sum_{l=1}^{n} \sigma_{lj}^2 \left(\sum_{k=1}^{r} U^{*}_{ik}U^{*}_{lk}\right)^2 \\
\rho &:= \frac{1}{p} \sum_{l=1}^{n} \E{E_{il}^3}  \left(\sum_{k=1}^{r} V^{*}_{lk} V^{*}_{jk}\right)^3 + \frac{1}{p}\sum_{l=1}^{n} \E{E_{lj}^3} \left(\sum_{k=1}^{r} U^{*}_{ik}U^{*}_{lk}\right)^3 \\
&\lesssim \frac{L^2}{p} n\left(\frac{\mu r}{n}\right)^3. 
\end{align*}
By \cite{esseen1942liapunov}, we have
\begin{align}
\sup_{t \in R}\left|P\left\{\frac{\epsilon^{(1)}_{ij}}{s_{ij}} \leq t\right\} - \Phi(t)\right| \lesssim s_{ij}^{-3} \rho \lesssim s_{ij}^{-3} \frac{L^2 \mu^3 r^{3}}{n^2p}. \label{eq:berry-esseen-bound}
\end{align}
Next, for any $t \in \R$, we want to bound $P\left\{\frac{\epsilon_{ij}^{(1)} + \epsilon_{ij}^{(2)}}{s_{ij}} \leq t\right\}.$ It is easy to check that
\begin{align*}
P\left\{\frac{\epsilon_{ij}^{(1)} + D_{ij}}{s_{ij}} \leq t\right\} - P(|\epsilon_{ij}^{(2)}| > D_{ij}) \leq P\left\{\frac{\epsilon_{ij}^{(1)} + \epsilon_{ij}^{(2)}}{s_{ij}} \leq t\right\} \leq P\left\{\frac{\epsilon_{ij}^{(1)} - D_{ij}}{s_{ij}} \leq t\right\} + P(|\epsilon_{ij}^{(2)}| > D_{ij}).
\end{align*}
Hence, 
\begin{align*}
&\left|P\left\{\frac{\epsilon_{ij}^{(1)} + \epsilon_{ij}^{(2)}}{s_{ij}} \leq t\right\} - \Phi(t)\right|  \\
&\leq \left|P\left\{\frac{\epsilon_{ij}^{(1)} \pm D_{ij}}{s_{ij}} \leq t\right\} - \Phi(t)\right|  + O(1/n^{10})\\
&\leq  \left|P\left\{\frac{\epsilon_{ij}^{(1)} \pm D_{ij}}{s_{ij}} \leq t\right\} - \Phi(t \pm D_{ij} / s_{ij})\right| + \left| \Phi(t + D_{ij} / s_{ij}) - \Phi(t - D_{ij}/s_{ij}) \right| + O(1/n^{10})\\
&\overset{(i)}{\lesssim } s_{ij}^{-3} \frac{L^2 \mu^3 r^{3}}{n^2p} +  \left| \Phi(t + D_{ij} / s_{ij}) - \Phi(t - D_{ij}/s_{ij})\right|  + O(1/n^{10})\\
&\overset{(ii)}{\lesssim}  s_{ij}^{-3} \frac{L^2 \mu^3 r^{3}}{n^2p} + \frac{D_{ij}}{s_{ij}} + O(1/n^{10})
\end{align*}
where (i) is due to \cref{eq:berry-esseen-bound}, and (ii) is due to the property of the standard Gaussian distribution $\Phi(\cdot).$ This completes the proof. 

Note that from the Bernstein inequality, one can also verify that $|\epsilon^{(1)}_{ij}| \lesssim \frac{\mu r L\log^{0.5}(n)}{\sqrt{np}}$ with probability $1-O(n^{-c}).$ This also implies an entry-wise error bound for $M^{\dd} - M^{*}.$
\begin{align}
|M^{\dd}_{ij} - M^{*}_{ij}| 
&\leq |\epsilon^{(1)}_{ij}| + |\epsilon^{(2)}_{ij}|\\
&\lesssim \frac{\mu r L\log^{0.5}(n)}{\sqrt{np}} + \frac{L^2\log^{3}(n)\mu r \kappa^5}{p\sigma_{\min}} + \frac{L\mu^2 r^2 \log^{2}(n)\kappa^4}{pn}\\
&\lesssim \frac{\mu r L\log^{0.5}(n)}{\sqrt{np}} \label{eq:entry-wise-bound}
\end{align}

\subsection{Proof of \cref{cor:empirical-inference}}
To begin, since $\sigma_{il}=\Theta(1)$, one can verify that
\begin{align*}
s_{ij}^2 \gtrsim \frac{\|U^{*}_{i,\cdot}\|^2+\|V^{*}_{j,\cdot}\|^2}{p}.
\end{align*}

We want to provide a bound for $|s_{ij}^2 - \hat{s}_{ij}^2|.$ To begin, with probability $1-O(n^{-c})$, consider
\begin{align*}
\left|\sum_{l\in [n]} \frac{1}{p} (E_{il}^2 \delta_{il} - \sigma_{il}^2p) \left(\sum_{k=1}^{r} V^{*}_{jk}V^{*}_{lk}\right)^2\right| &\lesssim \frac{1}{p}\sqrt{\sum_{l \in [n]}L^4 p  \|V^{*}_{j,\cdot}\|^4 \frac{\mu^2 r^2}{n^2}} \log(n)\\
&\lesssim L^2 \frac{1}{\sqrt{p}} \|V^{*}_{j,\cdot}\|^2 \frac{\mu r}{\sqrt{n}} \log(n)\\
&\lesssim s_{ij}^2 \sqrt{p} L^2 \frac{\mu r}{\sqrt{n}} \log(n).
\end{align*}
by the independence of $E, \delta$ and concentration bounds. Next, with probability $1-O(n^{-c})$, note that
\begin{align*}
\left|\sum_{l\in [n]} \frac{1}{p} (E_{il}^2 \delta_{il} - \hat{E}_{il}^2\delta_{il}) \left(\sum_{k=1}^{r} V^{*}_{jk}V^{*}_{lk}\right)^2\right| 
&\lesssim \left|\sum_{l\in [n]} \frac{1}{p} (E_{il}^2 \delta_{il} - \hat{E}_{il}^2\delta_{il}) \frac{\mu r}{n}\|V_{j,\cdot}\|^2\right|\\
&\overset{(i)}{\lesssim} \left|\sum_{l\in [n]} \frac{\delta_{il}}{p}  \frac{\mu r L^2\log^{0.5}(n)}{\sqrt{np}}  \frac{\mu r}{n}\|V_{j,\cdot}\|^2\right|\\
&\lesssim np\log(n) \frac{\mu r L^2\log^{0.5}(n)}{\sqrt{np}} \frac{\mu r}{n}\|V_{j,\cdot}\|^2\\
&\lesssim \sqrt{p} \frac{\mu^2 r^2 L^2 \log^{1.5}(n)}{\sqrt{n}}s_{ij}^2
\end{align*}
where (i) is due to \cref{eq:entry-wise-bound}. Then, consider, 
\begin{align*}
\left|\sum_{l\in [n]} \frac{1}{p} E_{il}^2 \delta_{il}\left(\sum_{k=1}^{r} V^{\dd}_{jk}V^{\dd}_{lk} - V^{*}_{jk}V^{*}_{lk}\right)^2\right| 
&\lesssim np\log^2(n) \frac{1}{p} L^2 \left(\norm{U^{\dd}-U^{*}}_{2,\infty} \norm{U^{*}}_{2,\infty}\right)^2\\
&\lesssim n\log^2(n) L^2 \left(\frac{\kappa L\log(n)}{\sigma_{\min}} \sqrt{\frac{n\log(n)}{p}} \frac{\mu r}{n}\right)^2\\
&\lesssim \frac{\kappa^2 L^4 \log^{5}(n) \mu^2 r^2}{\sigma_{\min}^2 p} \lesssim s_{ij}^2 \lesssim \frac{s_{ij}^2}{\log^3 n}.
\end{align*}
The above together implies that
\begin{align*}
|s_{ij}^2 - \hat{s}_{ij}^2| \lesssim  \frac{s_{ij}^2}{\log^3 n}.
\end{align*}
Let $\Delta = \frac{M^{\dd}_{ij} - M^{*}_{ij}}{s_{ij}} - \frac{M^{\dd}_{ij} - M^{*}_{ij}}{\hat{s}_{ij}} .$ Then, we can obtain
\begin{align*}
\left|\Delta\right| \lesssim \left|\frac{M^{\dd}_{ij} - M^{*}_{ij}}{s_{ij}}\right| \left|\frac{s^2_{ij} - \hat{s}_{ij}^2}{s_{ij}^2}\right| = o(1).
\end{align*}
We the finish the proof by the following.
\begin{align*}
P\left(\frac{M^{\dd}_{ij} - M^{*}_{ij}}{\hat{s}_{ij}} \leq t\right) 
&= P\left(\frac{M^{\dd}_{ij} - M^{*}_{ij}}{s_{ij}} \leq t - \Delta\right)\\
&= \Phi (t - \Delta) + o(1)\\
&\overset{(i)}{=} \Phi(t) + \Delta' + o(1)\\
&=\Phi(t) + o(1)
\end{align*}
where in (i), $|\Delta'| \lesssim |\Delta|.$

\section{Technical Lemmas}
\begin{lemma}\label{lem:reduction-ExptoGaussian}
Suppose $X$ is a zero-mean subexponential random variable with $\norm{X}_{\psi_1} \leq L.$ Then there exists a zero-mean bounded random variable $Y$ such that
\begin{align*}
    P(X \neq Y) \leq c_1/n^{10}
\end{align*}
and $|Y| \leq c_2 L\log n$ where $c_1,c_2$ are two constants.
\end{lemma}
\begin{proof}
Our idea for proving this lemma has two steps: (i) Truncation: we construct a random variable $Y' = X \cdot \1{|X| \leq k_1}$ for some $k_1$ and show that $E[Y'] \approx 0$. (ii) Slight modification: we then construct $Y$ from $Y'$ by slightly modifying the distribution to guarantee $E[Y]=0$.  

Without loss of generality, suppose $X$ is continuous. Let the density function of $X$ be $f_{X}.$ By continuity, there exists $k_1$ such that
\begin{align}
    \Pr(|X| > k_1) = \frac{2}{n^{10}}. \label{eq:control-k1-Prob}
\end{align}
Let $Y' = X \cdot \1{X \leq k_1}.$ We aim to provide a bound on $E[Y'].$

Since $X$ is centered and subexponential, $\Pr(|X| > t) \leq 2\exp(-tC/L)$ for some constant $C$. Therefore, 
\begin{align*}
    \frac{2}{n^{10}} \leq 2\exp(-k_1 C/L) \implies k_1 \leq \frac{10L\log(n)}{C}.
\end{align*}
Let $k_2 \triangleq \frac{10L \log(n)}{C}.$ Note that
\begin{align*}
    E[X] 
    &= \int_{-\infty}^{\infty} x f_{X}(x) dx \\
    &= \int_{0}^{\infty} x \left(f_{X}(x) - f_{X}(-x)\right) dx\\
    &= \int_{0}^{k_1} x\left(f_{X}(x) - f_{X}(-x)\right) dx + \int_{k_1}^{k_2} x\left(f_{X}(x) - f_{X}(-x)\right) dx + \int_{k_2}^{\infty} x\left(f_{X}(x) - f_{X}(-x)\right) dx.
\end{align*}
Note that $E[Y'] = \int_{0}^{k_1} x\left(f_{X}(x) - f_{X}(-x)\right) dx.$ Since $E[X] = 0$, then
\begin{align}
    \left|E[Y']\right| 
    &\leq \left|\int_{k_1}^{k_2} x\left(f_{X}(x) - f_{X}(-x)\right) dx\right| + \left|\int_{k_2}^{\infty} x\left(f_{X}(x) - f_{X}(-x)\right) dx\right|\\
    &\leq k_2 \int_{k_1}^{k_2} (f_X(x) + f_X(-x))dx + \int_{k_2}^{\infty} x\left(f_{X}(x) + f_{X}(-x)\right) dx\\
    &\leq k_2 \int_{k_1}^{\infty} (f_X(x) + f_X(-x))dx + \int_{k_2}^{\infty} (x-k_2)\left(f_{X}(x) + f_{X}(-x)\right) dx\\
    &\leq k_2 \Pr(|X| > k_1) + \int_{k_2}^{\infty} (x-k_2)\left(f_{X}(x) + f_{X}(-x)\right) dx. \label{eq:k1-part}
\end{align}
We need to use the subexponential property to bound $\int_{k_2}^{\infty} (x-k_2)\left(f_{X}(x) + f_{X}(-x)\right) dx$. In particular, 
\begin{align}
    \int_{k_2}^{\infty} (x-k_2)\left(f_{X}(x) + f_{X}(-x)\right) dx 
    &= \int_{x=k_2}^{\infty} \left(f_{X}(x) + f_{X}(-x)\right) \left(\int_{t=0}^{\infty} \1{t+k_2\leq x} dt\right) dx\\
    &= \int_{t=0}^{\infty} \left(\int_{x=k_2}^{\infty} \1{t+k_2\leq x} \left(f_{X}(x) + f_{X}(-x)\right) dx\right)dt\\
    &= \int_{t=0}^{\infty} \left(\int_{x=t+k_2}^{\infty} \left(f_{X}(x) + f_{X}(-x)\right) dx\right)dt\\
    &= \int_{t=0}^{\infty} \Pr(|X|>t+k_2) dt\\
    &\leq \int_{t=k_2}^{\infty} 2\exp(-tC/L) dt\\
    &= 2\frac{-L}{C}\exp(-tC/L) \Big \rvert_{t=k_2}^{\infty} \\
    &= 2\frac{L}{C} \exp(-k_2C/L). \label{eq:k2-part}
\end{align}
Combining \cref{eq:control-k1-Prob,eq:k1-part,eq:k2-part}, we have
\begin{align}
    |E[Y']|
    &\leq \frac{2k_2}{n^{10}} + 2\frac{L}{C} \exp(-k_2C/L) \\
    &\leq \frac{2k_2}{n^{10}} + 2\frac{L}{C} \frac{1}{n^{10}}\\
    &\leq \frac{3 k_2}{n^{10}}. \label{eq:p-label}
\end{align}

Let 
\begin{align}
    p \triangleq \frac{|E[Y']|}{3k_2} / \frac{1}{n^{10}}.
\end{align}
Then $p \in [0, 1]$ by \cref{eq:p-label}.

Now the zero-mean random variable $Y$ can be constructed. Let $Z\sim \Ber(p)$ be independent from $X$. Denote
\begin{align}
    Y \triangleq \begin{cases}X & |X| \leq k_1 \\ -\text{sign}(E[Y'])3k_2 & |X| > k_1 \text{~and~} Z = 1 \\ 0 & \text{otherwise}\end{cases}
\end{align}
Then 
\begin{align*}
    E[Y] &= \left(\int_{0}^{k_1} x\left(f_{X}(x) - f_{X}(-x)\right) dx\right) + \Pr(|X|>k_1 \text{~and~} Z = 1) (-3\text{sign}(E[Y'])k_2)\\
    &= E[Y'] - \frac{1}{n^{10}} p (3k_2) \text{sign}(E[Y'])\\
    &= 0.
\end{align*}
Also, $|Y| \leq 3k_2 = 30\frac{L}{C}\log(n)$ is bounded. Furthermore, with probability $1-\frac{2}{n^{10}}$, $|X| \leq k_1$ and $Y=X.$ 
\end{proof}

\begin{lemma}[Bernstein's inequality (Theorem 2.8.4, \cite{vershynin2018high})]\label{lem:Bernstein}
Let $X_1, X_2, \dotsc, X_N$ be independent zero-mean random variables, such that $|X_i| \leq K$ for all $i$. Then, for every $t \geq 0$, we have 
\begin{align}
    \Prob{\left|\sum_{i=1}^{N} X_i\right| \geq t} \leq 2\exp\left(-\frac{t^2/2}{\sigma^2 + Kt/3}\right)
\end{align}
here $\sigma^2 = \sum_{i=1}^{N} \E{X_i^2}$ is the variance of the sum.
\end{lemma}

\begin{lemma}
Let $X_1, X_2, \dotsc, X_n$ be independent zero-mean random variables, such that $\norm{X_i}_{\psi_1} \leq L$ for all $i$. Then, with probability $1-O(n^{-c})$ for any constant $c$, we have
\begin{align}
     \left|\sum_{i=1}^{n}X_i \right| \lesssim \sigma\sqrt{\log(n)} + L\log^2(n)
\end{align}
here $\sigma^2 = \sum_{i=1}^{n} \E{X_i^2}$ is the variance of the sum.
\end{lemma}
\begin{proof}
Let $Y_i = X_i \1{|X_i| \leq B}$ be the truncated version of $X_i$. Then $Var(Y_i) \leq E[Y_i^2] \leq E[X_i^2].$ Furthermore, 
\begin{align}
    |\E{Y_i}| 
    &\leq \left| \int_{B}^{\infty} X_i df(X_i) + \int_{-\infty}^{-B} X_i df(X_i)\right|\\
    &\leq B P(|X_i| >B) + \int_{B}^{\infty} P(|X_i| > B) dX\\
    &\leq Be^{-B/CL} + CL e^{-B/CL} \label{eq:Ex-Yi-bound}
\end{align}
where $C$ is a constant. By \cref{lem:Bernstein}, we have
\begin{align}
    \Prob{\left|\sum_{i=1}^{N}(Y_i - \E{Y_i})\right| \geq t} \leq 2\exp\left(-\frac{t^2/2}{\sum_{i=1}^{N} Var(Y_i) + Bt/3}\right) \leq 2\exp\left(-\frac{t^2/2}{\sigma^2 + Bt/3}\right).
\end{align}
Then, with probability $1-O(n^{-c})$ for some constant $c$, 
\begin{align}
    \left|\sum_{i=1}^{N}(Y_i - \E{Y_i})\right| \lesssim \sigma\sqrt{\log(n)} + B\log(n).
\end{align}
Take $B = L\log(n)C'$ for a proper constant $C'$, by \cref{eq:Ex-Yi-bound}, we have
\begin{align}
    \left|\sum_{i=1}^{N}Y_i \right| \lesssim \sigma\sqrt{\log(n)} + L\log^2(n).
\end{align}
By the union bound on the event $|X_i| \leq B$ for all $i$, we can conclude that, with probability $1-O(n^{-c})$, 
\begin{align}
    \left|\sum_{i=1}^{N}X_i \right| \lesssim \sigma\sqrt{\log(n)} + L\log^2(n).
\end{align}
\end{proof}

\begin{lemma}[Matrix Bernstein inequality, Theorem 6.1.1 \cite{tropp2015introduction}]\label{lem:matrix-Bernstein}
Given $n$ independent random $m_1 \times m_2$ matrices $X_1, X_2, \dotsc, X_{n}$ with $E[X_i] = 0.$ Let 
\begin{align}
    V \triangleq max\left(\norm{\sum_{i=1}^{n} E[X_i X_i^{\top}]}, \norm{\sum_{i=1}^{n} E[X_i^{\top} X_i]}\right).
\end{align}
Suppose $\norm{X_i} \leq B$ for $i \in [n].$ For all $t\geq 0$, 
\begin{align}
    \Prob{\norm{X_1+X_2+\dotsc + X_{n}} \geq t} \leq (m_1+m_2) \exp\left\{-\frac{t^2 / 2}{V + Bt/3}\right\}.
\end{align}
\end{lemma}

\begin{lemma}[Perturbation of pseudo-inverses (Theorem 3.3, \cite{stewart1977perturbation})]\label{lem:pseudo-inverse}
Let $A^{-1}$ and $B^{-1}$ be the pseudo-inverse (Moore-Penrose inverse) of two matrices $A$ and $B$, respectively. Then
\begin{align}
    \norm{B^{-1}-A^{-1}} \leq 3 \max\left({\norm{A^{-1}}^2, \norm{B^{-1}}^2}\right) \norm{B-A}.
\end{align}
\end{lemma}

\begin{lemma}[Lemma 4 \cite{mcrae2019low}]\label{lem:sub-exponential-operator-norm}
Let $X \in R^{m \times n}$ be a random matrix whose entries are independent and centered, and suppoes that for some $v, t_0 > 0$, we have, for all $t \geq t_0$,
\begin{align*}
    \Pr(|X_{ij}| \geq t) \leq 2e^{-t/v}.
\end{align*}
Let $\epsilon \in (0, 1/2)$, and let
\begin{align*}
    K = \max\{t_0, v\log\frac{2m n}{\epsilon}\}.
\end{align*}
Then
\begin{align*}
    \Pr\left(\norm{X} \geq 2\sigma + \frac{\epsilon v}{\sqrt{m n}} + t \right) \leq \max(m, n) \exp\left(-\frac{t^2}{C_0 2K^2}\right) + \epsilon
\end{align*}
where $C_0$ is a constant and 
\begin{align*}
    \sigma = \max_{i} \sqrt{\sum_{j} \E{X_{ij}^2}} + \max_{j} \sqrt{\sum_{i} \E{X_{ij}^2}}.
\end{align*}
\end{lemma}

\subsection{Proof of \cref{lem:subexponential-Bernstein} and \cref{lem:operator-norm-bound}}
\begin{proof}[Proof of \cref{lem:subexponential-Bernstein}]
Let $Y_i = X_i \1{\norm{X_i} \leq B}$ be the truncated version of $X_i$. We have, 
\begin{align*}
    \norm{\E{Y_i}} 
    &\leq \norm{\int X_i \1{\norm{X_i} > B} df(X_i)} \\
    &\overset{(i)}{\leq} \int \norm{X_i} \1{\norm{X_i} > B} df(X_i)\\
    &\leq B P(\norm{X_i} > B) + \int_{B}^{\infty} P(\norm{X_i} > t) dt\\
    &\overset{(ii)}{\leq} Be^{-B/CL} + CL e^{-B/CL} \label{eq:Ex-Yi-bound}
\end{align*}
where (i) is due to the convexity of $\norm{\cdot}$ and (ii) is due to the subexponential property of $\norm{X_i}$ and $C$ is a constant. Meanwhile, we have
\begin{align*}
    \norm{\sum_{i=1}^{n} \E{(Y_i-\E{Y_i}) (Y_i-\E{Y_i})^{\top}}} 
    &= \norm{\sum_{i=1}^{n} \E{Y_iY_i^{\top}} - \E{Y_i}\E{Y_i}^{\top}}\\
    &\overset{(i)}{\leq} \norm{\sum_{i=1}^{n} \E{Y_iY_i^{\top}}}\\
    &= \norm{\sum_{i=1}^{n} \E{X_iX_i^{\top}} - \E{X_iX_i^{\top}\1{\norm{X_i}>B}}}\\
    &\overset{(ii)}{\leq} \norm{\sum_{i=1}^{n} \E{X_iX_i^{\top}}} \leq V
\end{align*}
where (i) is due to the positive-semidefinite property of $\E{Y_i}\E{Y_i}^{\top}$ and $\E{Y_iY_i^{\top}} - \E{Y_i}\E{Y_i}^{\top}$, (ii) is due to the positive-semidefinite property of $\E{X_iX_i^{\top}\1{\norm{X_i}>B}}$ and $\E{Y_iY_i^{\top}}.$ Similarly, $\norm{\sum_{i=1}^{n} \E{(Y_i-\E{Y_i})^{\top} (Y_i-\E{Y_i})}} \leq V.$

Then, by \cref{lem:matrix-Bernstein}, we have
\begin{align*}
    \Prob{\norm{\sum_{i=1}^{N}(Y_i - \E{Y_i})} \geq t} \leq 2\exp\left(-\frac{t^2/2}{V + 2Bt/3}\right).
\end{align*}
Then, with probability $1-O(n^{-c})$ for some constant $c$, 
\begin{align*}
    \norm{\sum_{i=1}^{N}(Y_i - \E{Y_i})} \lesssim \sqrt{V\log(n(m_1+m_2))} + B\log(n(m_1+m_2)).
\end{align*}
Take $B = L\log(n)C'$ for a proper constant $C'$, by \cref{eq:Ex-Yi-bound}, we have
\begin{align*}
    \norm{\sum_{i=1}^{N}Y_i} 
    &\lesssim \sqrt{V\log(n)} + L\log^2(n) + nL\log(n)O(n^{-C'/C})\\ 
    &\lesssim \sqrt{V\log(n(m_1+m_2))} + L\log(n(m_1+m_2))\log(n).
\end{align*}
By the union bound on the event $\norm{X_i} \leq B$ for all $i$, we can conclude that, with probability $1-O(n^{-c'})$ for some constant $c'$, 
\begin{align*}
    \norm{\sum_{i=1}^{N}X_i} \lesssim  \sqrt{V\log(n(m_1+m_2))} + L\log(n(m_1+m_2))\log(n).
\end{align*}
\end{proof}

\begin{proof}[Proof of \cref{lem:operator-norm-bound}]
We invoke \cref{lem:sub-exponential-operator-norm} to prove this result. Let $C_1, C_2, C_3$ be constants. Let $X = P_{\Omega}(E)$ with $\epsilon = \frac{1}{n^{11}}, v= C_1 L, K = C_1L\log(n)$. It is easy to verify that $E[X_{ij}^2] \leq pL^2.$ Therefore, $\sigma \leq C_2\sqrt{np} L.$ Take $t = C_3\sqrt{np}L$. By \cref{lem:sub-exponential-operator-norm}, 
\begin{align*}
    \Prob{\norm{P_{\Omega}(E)} \geq 2C_3(\sqrt{np}L + \frac{L}{n^{11}})} \leq C_4 n e^{-np/\log^2(n)} + \frac{1}{n^{11}}.
\end{align*}
Given that $np \geq c_0 log^3 n$, we have, with probability $1-O(n^{-11})$,
\begin{align*}
    \norm{P_{\Omega}(E)} \lesssim L\sqrt{np}. 
\end{align*}
\end{proof}

\end{appendices}

\end{document}